\documentclass[10pt,twocolumn,letterpaper]{article}

\usepackage[pagenumbers]{cvpr} % To force page numbers, e.g. for an arXiv version

\newcommand{\calN}{\mathcal{N}}
\newcommand{\calO}{\mathcal{O}}

\newcommand{\calS}{\mathcal{S}}

\newcommand{\ba}{\mathbf{a}}

\newcommand{\bc}{\mathbf{c}}

\newcommand{\bg}{\mathbf{g}}

\newcommand{\bI}{\mathbf{I}}

\newcommand{\bx}{\mathbf{x}}

\newcommand{\rmd}{\mathrm{d}}

\DeclareMathAlphabet{\mathbsf}{OT1}{cmss}{bx}{n}
\DeclareMathAlphabet{\mathssf}{OT1}{cmss}{m}{sl}% slanted sans serif

\DeclareSymbolFont{bsfletters}{OT1}{cmss}{bx}{n}  
\DeclareSymbolFont{ssfletters}{OT1}{cmss}{m}{n}
\DeclareMathSymbol{\bsfGamma}{0}{bsfletters}{'000}
\DeclareMathSymbol{\ssfGamma}{0}{ssfletters}{'000}
\DeclareMathSymbol{\bsfDelta}{0}{bsfletters}{'001}
\DeclareMathSymbol{\ssfDelta}{0}{ssfletters}{'001}
\DeclareMathSymbol{\bsfTheta}{0}{bsfletters}{'002}
\DeclareMathSymbol{\ssfTheta}{0}{ssfletters}{'002}
\DeclareMathSymbol{\bsfLambda}{0}{bsfletters}{'003}
\DeclareMathSymbol{\ssfLambda}{0}{ssfletters}{'003}
\DeclareMathSymbol{\bsfXi}{0}{bsfletters}{'004}
\DeclareMathSymbol{\ssfXi}{0}{ssfletters}{'004}
\DeclareMathSymbol{\bsfPi}{0}{bsfletters}{'005}
\DeclareMathSymbol{\ssfPi}{0}{ssfletters}{'005}
\DeclareMathSymbol{\bsfSigma}{0}{bsfletters}{'006}
\DeclareMathSymbol{\ssfSigma}{0}{ssfletters}{'006}
\DeclareMathSymbol{\bsfUpsilon}{0}{bsfletters}{'007}
\DeclareMathSymbol{\ssfUpsilon}{0}{ssfletters}{'007}
\DeclareMathSymbol{\bsfPhi}{0}{bsfletters}{'010}
\DeclareMathSymbol{\ssfPhi}{0}{ssfletters}{'010}
\DeclareMathSymbol{\bsfPsi}{0}{bsfletters}{'011}
\DeclareMathSymbol{\ssfPsi}{0}{ssfletters}{'011}
\DeclareMathSymbol{\bsfOmega}{0}{bsfletters}{'012}
\DeclareMathSymbol{\ssfOmega}{0}{ssfletters}{'012}

\usepackage{thmtools, thm-restate}
\newtheorem{theorem}{Theorem} 
\newtheorem{lemma}[theorem]{Lemma}

\newenvironment{proof}[1][Proof]{\begin{trivlist}
\item[\hskip \labelsep {\bfseries #1}]}{\end{trivlist}}
\usepackage[dvipsnames]{xcolor}
\definecolor{cvprblue}{rgb}{0.21,0.49,0.74}
\usepackage[pagebackref,breaklinks,colorlinks,citecolor=cvprblue]{hyperref}
\usepackage{algorithm}
\usepackage{algorithmic}
\usepackage{multirow}
\usepackage{nicefrac}

\def\paperID{5399} % *** Enter the Paper ID here
\def\confName{CVPR}
\def\confYear{2024}

\title{Towards Accurate Guided Diffusion Sampling through  \\  Symplectic Adjoint Method}

\author{%
  Jiachun Pan$^*$ \\ 
    National University of Singapore\\ {\tt\small pan.jc@nus.edu.sg} \\
    \and
  Hanshu Yan$^*$\\
  ByteDance \\
  {\tt\small hanshu.yan@bytedance.com} \\
  \and
  Jun Hao Liew \\
  ByteDance \\
  {\tt\small junhao.liew@bytedance.com} \\
  \and
  Jiashi Feng \\
  ByteDance \\
  {\tt\small jshfeng@bytedance.com} \\
  \and
  Vincent Y. F. Tan \\
  National University of Singapore \\
  {\tt\small vtan@nus.edu.sg} \\
}

\begin{document}
% \maketitle

\twocolumn[{%
\renewcommand\twocolumn[1][]{#1}%
\maketitle
\begin{center}
    \centering
    \captionsetup{type=figure}
    \includegraphics[width=\textwidth]{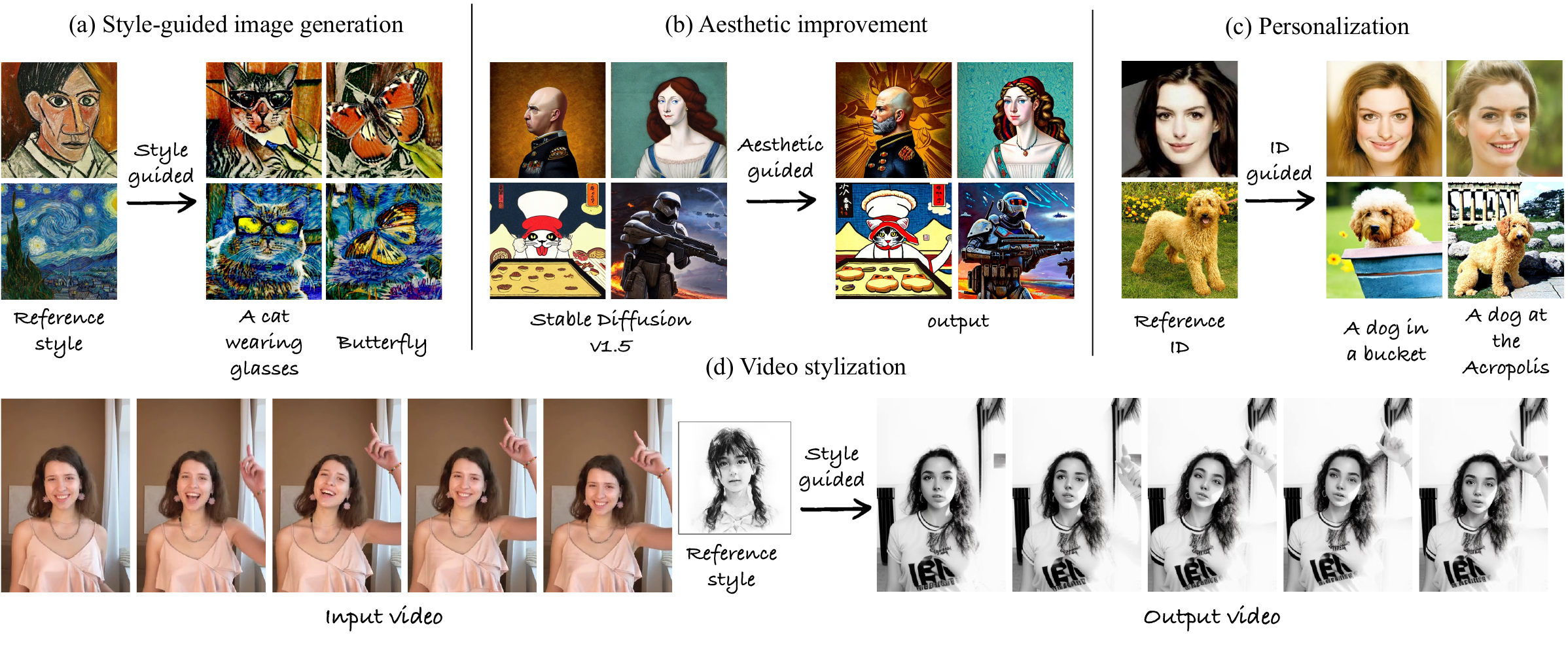}
    \captionof{figure}{We propose \textbf{Symplectic Adjoint Guidance}, a training-free guided diffusion process that supports various image and video generation tasks, including style-guided image generation, aesthetic improvement, personalization and video stylization.}
    \label{fig:intro}
\end{center}%
}]
\def\thefootnote{*}\footnotetext{Equal contribution.}\def\thefootnote{\arabic{footnote}}

\begin{abstract}
Training-free guided sampling in diffusion models leverages off-the-shelf pre-trained networks, such as an aesthetic evaluation model, to guide the generation process. Current training-free guided sampling algorithms obtain the guidance energy function based on a one-step estimate of the clean image. However, since the off-the-shelf pre-trained networks are trained on clean images, the one-step estimation procedure of the clean image may be inaccurate, especially in the early stages of the generation process in diffusion models. This causes the guidance in the early time steps to be inaccurate. To overcome this problem, we propose  Symplectic Adjoint Guidance (SAG), which calculates the gradient guidance in two inner stages. Firstly, SAG
estimates the clean image via $n$  function calls, where $n$ serves as a flexible hyperparameter that can be tailored to meet specific image quality requirements. Secondly, SAG uses the symplectic adjoint method to obtain the gradients accurately and efficiently in terms of the memory requirements.  Extensive experiments demonstrate that SAG generates images with higher qualities compared to the baselines in both guided image and video generation tasks. Code is
available at \url{https://github.com/HanshuYAN/AdjointDPM.git}
\end{abstract}    
\vspace{-0.5em}
\section{Introduction}
\label{sec:intro}

\begin{figure*}
    \centering
    \includegraphics[width=0.75\textwidth]{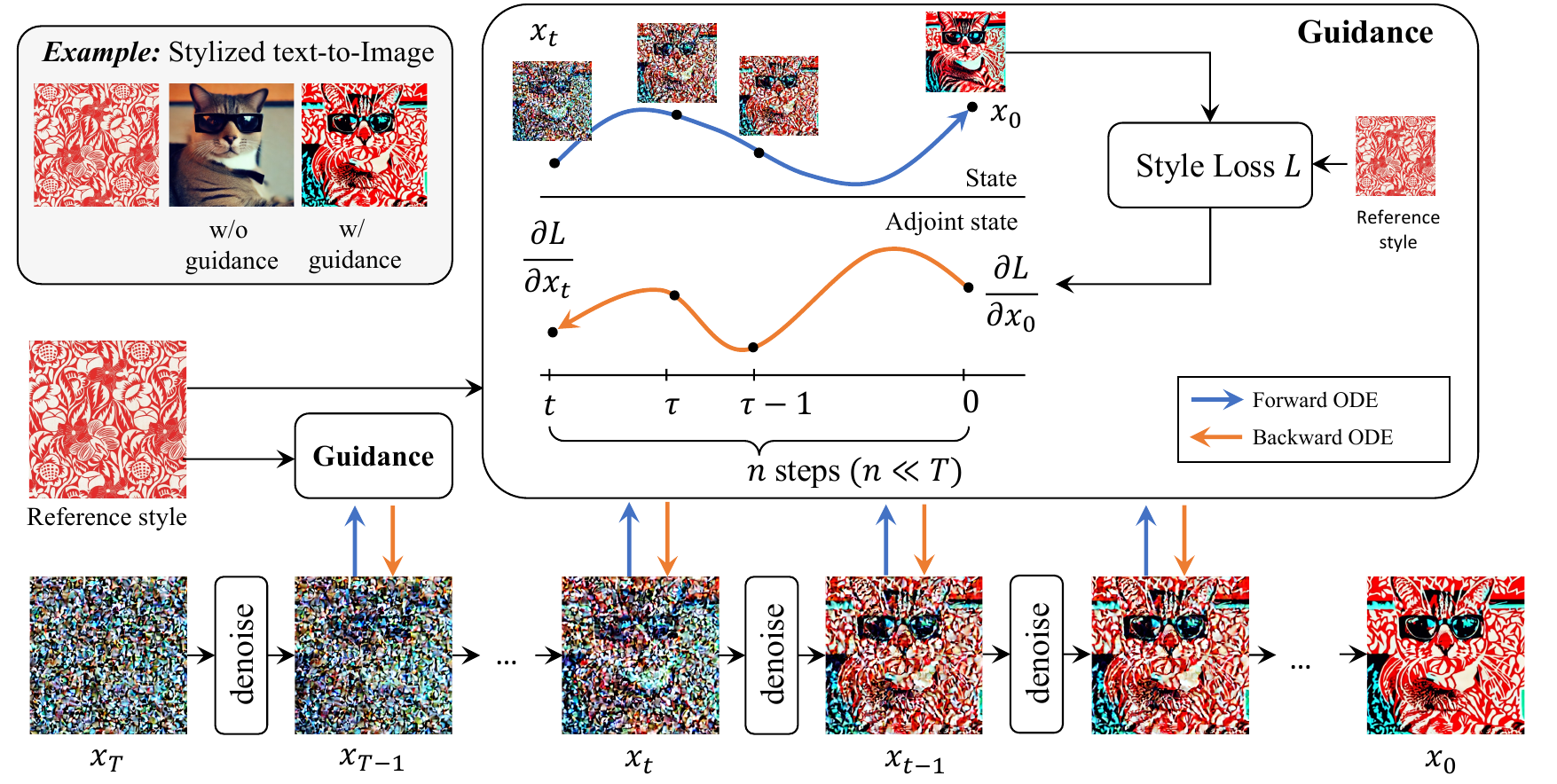}
    \caption{Symplectic Adjoint Training-free Guidance Generation. We illustrate the framework of training-free guided generation through symplectic adjoint guidance using a stylization example. When we denoise  Gaussian noise to an image across various steps, we can add guidance (usually defined as gradients of loss function on the estimate of $\hat{\bx}_0$ based on $\bx_t$) to each step. Different from previous works~\cite{Yu2023freedom,Bansal_2023_CVPR} which approximate $\hat{\bx}_0$ based on $\bx_t$ using one step, we estimate $\hat{\bx}_0$ using $n$ steps $(n\ll T)$ by solving a forward ODE. Then we use the symplectic adjoint method to solve a backward ODE to obtain the gradients. These gradients guide the diffusion generation process to be closer to the reference image.}
    \label{fig:intro}
    \vspace{-0.5em}
\end{figure*}

% \begin{figure*}
%   \centering
%   \begin{subfigure}{0.68\linewidth}
%     \fbox{\rule{0pt}{2in} \rule{.9\linewidth}{0pt}}
%     \caption{An example of a subfigure.}
%     \label{fig:short-a}
%   \end{subfigure}
%   \hfill
%   \begin{subfigure}{0.28\linewidth}
%     \fbox{\rule{0pt}{2in} \rule{.9\linewidth}{0pt}}
%     \caption{Another example of a subfigure.}
%     \label{fig:short-b}
%   \end{subfigure}
%   \caption{Example of a short caption, which should be centered.}
%   \label{fig:short}
% \end{figure*}

Diffusion models are powerful generative models that exhibit impressive performances across different modality generation, including image~\cite{dhariwal2021diffusion, ho_imagen_2022,ho_classifier-free_2022}, video~\cite{molad_dreamix_2023,wu_tune--video_2023,zhou_magicvideo_2022} and audio generation~\cite{pmlr-v202-liu23f}. Guided sampling, including classifier guidance~\cite{dhariwal2021diffusion} and classifier-free guidance~\cite{ho_classifier-free_2022}, has been widely used in diffusion models to realize controllable generation, such as text-to-image generation~\cite{saharia2022photorealistic}, image-to-image generation~\cite{saharia2022palette,parmar2023zero}, and ControlNet~\cite{zhang2023adding}. Guided sampling controls the outputs of generative models by conditioning on various types of signals, such as descriptive text, class labels, and images. 

A line of guidance methods involves task-specific training of diffusion models using paired data, \ie, targets and conditions. For instance, classifier guidance~\cite{dhariwal2021diffusion} combines the score estimation of diffusion models with the gradients of the image classifiers to direct the generation process to produce images corresponding to a particular class. In this way, several image classifiers need to be trained on the noisy states of intermediate generation steps of diffusion models. Alternatively, classifier-free guidance~\cite{ho_classifier-free_2022} directly trains a new score estimator with conditions and uses a linear combination of conditional and unconditional score estimators for sampling. Although this line of methods can effectively guide diffusion models to generate data satisfying certain properties, they are not sufficiently flexible to adapt to any type of guiding due to the cost of training and the feasibility of collecting paired data.

To this end, another line of training-free guidance methods has been explored~\cite{Bansal_2023_CVPR, Yu2023freedom,li2022upainting}.  In training-free guided sampling, at a certain sampling step $t$, the guidance function is usually constructed as the gradients of the loss function obtained by the off-the-shelf pre-trained models, such as face-ID detection or aesthetic evaluation models. More specifically, 
% as the off-the-shelf pre-trained models are often trained on clean data, 
the guidance gradients are computed based on the one-step approximation of denoised images from the noisy samples at certain steps $t$. Then,  gradients are added to corresponding sampling steps as guidance to direct the generation process to the desired results. This line of methods offers greater flexibility by allowing the diffusion models to adapt to a broad spectrum of guidance.
However, at certain time steps with guidance, the generated result at the end is usually misaligned with its one-step denoising approximation, which may lead to inaccurate guidance. The misalignment is notably pronounced in the early steps of the generation process, as the noised samples are far from the finally generated result. For example, in face ID-guided generation, when the final approximation is blurry and passed to pre-trained face detection models, we cannot obtain accurate ID features, which leads to inaccuracies in guidance to the desired faces.

% \HS{in this section, we talk about the motivation to explore the multiple-step estimation of clean images. When we say "inaccurate", we can add some real examples, such as the color fading, the bad texture...}

To mitigate the misalignment issue of existing training-free guidance, we propose a novel guidance algorithm, termed Sympletic Adjoint Guidance (SAG). As shown in Figure~\ref{fig:intro}, SAG estimates the finally generated results by \textbf{$\textbf{n}$-step} denoising. Multiple-step estimation yields more accurate generated samples, but this also introduces another challenge in backpropagating the gradients from the output to each intermediate sampling step. Because the execution of the vanilla backpropagation step requires storing all the intermediate states of the $n$ iterations, the memory cost is prohibitive.   %It is due to the vanilla gradient backpropagation needs to store all the intermediate states of the $n$-times of iterations, which results in a huge memory assumption. 
To tackle this challenge, SAG applies the \textbf{symplectic adjoint method}, an adjoint method solved by a symplectic integrator~\cite{matsubara2021symplectic}, which can backpropagate the gradients accurately and is memory efficient. 
% \HS{one sentence to explain what is Symplectic Adjoint method and why it's accurate and memory efficient}
In summary, our contributions are as follows:
\begin{itemize}
    \item We propose to use an $n$-step estimate of the final generation to calculate the gradient guidance. This mitigates the misalignment between the final outputs and their estimates, which provides more accurate guidance from off-the-shelf pre-trained models. 
    \item To backpropagate gradients throughout the $n$-step estimate, we introduce the theoretically grounded symplectic adjoint method to obtain accurate gradients. This method is also memory efficient, which is beneficial to guided sampling in large models, such as Stable Diffusion. 
    \item Thanks to accurate guidance, SAG can obtain high-quality results in various guided image and video generation tasks.
\end{itemize}
% \JH{mention that thanks to the accurate guidance, our method supports various image and video generation tasks, including xxx}
\section{Background}
\label{sec:back}

\subsection{Guided Generation in Diffusion Models}
\label{subsec:tfree}
\vspace{-0.3em}
\paragraph{Diffusion Models} Diffusion generative models gradually add Gaussian noise to complex data distributions to transform them into a simple Gaussian distribution and then solve the reverse process to generate new samples. The forward noising process and reverse denoising process can both be modeled as SDE and ODE forms~\cite{song2020score}. 
In this paper, we mainly consider the ODE form of diffusion models as it is a deterministic method for fast sampling of diffusion models. An example for discrete deterministic sampling (solving an ODE) is  DDIM~\cite{song_denoising_2022}, which has the form:
\begin{equation}
\label{eqn:ddim}
\setlength\abovedisplayskip{6pt}
\setlength\belowdisplayskip{6pt}
    \bx_{t-1} = \sqrt{\alpha_{t-1}} \hat{\bx}_0 + \sqrt{1-\alpha_{t-1}} \epsilon_{\theta}(\bx_t,t),
\end{equation}
where $\alpha_t$ is a schedule that controls the degree of diffusion at each time step, $\epsilon_{\theta}(\bx_t,t)$ is a network that predicts noise, and $\hat{\bx}_0$ is an estimate of the clean image:
\begin{equation}
\label{eqn:estx0}
\setlength\abovedisplayskip{6pt}
\setlength\belowdisplayskip{6pt}
\hat{\bx}_0=\frac{ \bx_t-\sqrt{1-\alpha_t}\epsilon_{\theta}(\bx_t,t)}{\sqrt{\alpha_t}}.
\end{equation}
The DDIM can be regarded as the discretization of an ODE. By multiplying both sides of~\eqref{eqn:ddim} with $\sqrt{{1}/{\alpha_{t-1}}}$, we have
\begin{align*}
    \frac{\bx_{t-1}}{\sqrt{\alpha_{t-1}}} = \frac{\bx_{t}}{\sqrt{\alpha_{t}}} +\epsilon_{\theta}(\bx_t,t) \left(\frac{\sqrt{1-\alpha_{t-1}}}{\sqrt{\alpha_{t-1}}}-\frac{\sqrt{1-\alpha_t}}{\sqrt{\alpha_t}}\right).
\end{align*}
We  can parameterize $\sigma_t = \sqrt{1-\alpha_t}/\sqrt{\alpha_t}$ as $\sigma_t$ is monotone in  $t$~\cite{song_denoising_2022} and $\bar{\bx}_{\sigma_t}=\bx_t/\sqrt{\alpha_t}$. 
% Since $\sigma_t$ is monotonical with respect to $t$, 
Then when $\sigma_{t-1}-\sigma_t\to 0$, we obtain the ODE form of DDIM: 
\begin{equation}
\setlength\abovedisplayskip{6pt}
\setlength\belowdisplayskip{6pt}
\rmd \bar{\bx}_{\sigma_t}= \bar{\epsilon} (\bar{\bx}_{\sigma_t},\sigma_t)\rmd \sigma_t\label{eqn:ode1},
\end{equation}
where $\bar{\epsilon} (\bar{\bx}_{\sigma_t},\sigma_t)=\epsilon_{\theta}(\bx_t,t)$. Using ODE forms makes it possible to use numerical methods to accelerate the sampling process~\cite{lu2022dpm2}.
% and in~\eqref{eqn:ode2}, $\tilde{\sigma}=\sqrt{\alpha_t}/\sqrt{1-\alpha_t}, \tilde{\bx}=\bx_t/\sqrt{1-\alpha_t}$ and $\tilde{s}(\tilde{\bx},\tilde{\sigma}) = (\bx_t - \sqrt{1-\alpha_t}\epsilon_{\theta}(\bx_t,t))/\sqrt{\alpha_t}$ that is the estimation of the clean image. The ODE~\eqref{eqn:ode2} has the advantage of keeping the differentiation bounded within the pixel value range in pixel-based diffusion~\cite{lu2022dpm2}.
\vspace{-0.5em}
\paragraph{Guided Generation}
Guided sampling in diffusion models can roughly be divided into two categories: training-required and training-free.
% \paragraph{Training-based guidance}
Training-required models~\cite{ho_classifier-free_2022,song_denoising_2022,dhariwal2021diffusion,rombach_high-resolution_2022} are usually well-trained on paired data of images and guidance, leading to strong guidance power in the diffusion sampling process. However, they lack flexibility in adapting  to a variety of guidances.

In this work, we mainly focus on the training-free guided sampling in diffusion models~\cite{meng2021sdedit, parmar2023zero, Bansal_2023_CVPR, hertz2022prompt, Yu2023freedom}.
% There are two typical ways -- classifier and classifier-free guidance. In classifier guidance~\cite{dhariwal2021diffusion}, it combines the score estimate of a diffusion model with the gradient of a noise-aware image classifier, which requires training an image classifier separate from the diffusion model on noisy data. In classifier-free guidance~\cite{ho_classifier-free_2022}, a diffusion model is jointly trained with a conditional and an unconditional score estimate. The resulting conditional and unconditional score estimates are then combined to attain a trade-off between sample quality and diversity similar to that obtained using classifier guidance.
% \paragraph{Training-free guidance}
% Training-free guided sampling~\cite{meng2021sdedit, parmar2023zero, Bansal_2023_CVPR, hertz2022prompt, Yu2023freedom} in diffusion models usually
Training-free guided sampling methods, such as FreeDOM~\cite{Yu2023freedom} and Universal Guidance (UG)~\cite{Bansal_2023_CVPR} leverage the off-the-shelf pre-trained networks to guide the generation process. To generate samples given some conditions $\bc$, a guidance function is added to the diffusion ODE function:
\begin{equation}
\setlength\abovedisplayskip{6pt}
\setlength\belowdisplayskip{6pt}
    \frac{\rmd \bar{\bx}_{\sigma_t}}{\rmd \sigma_t} = \bar{\epsilon} (\bar{\bx}_{\sigma_t}, \sigma_t) + \rho_{\sigma_t}g(\bar{\bx}_{\sigma_t},\bc,\sigma_t),
\end{equation}
where $\rho_{\sigma_t}$ is the parameter that controls the strength of guidance and $g(\bar{\bx}_{\sigma_t},\bc,\sigma_t)$ is usually taken as the negative gradients of loss functions $-\nabla_{\bar{\bx}_{\sigma_t}}L(\bar{\bx}_{\sigma_t},\bc)$~\cite{Yu2023freedom,Bansal_2023_CVPR} obtained by the off-the-shelf networks. For example, in the stylization task, $L$ could be the style loss between $\bar{\bx}_{\sigma_t}$ and the style images. As the off-the-shelf networks are trained on clean data, directly using them to obtain the loss function of noisy data $\bar{\bx}_{\sigma_t}$ is improper. To address this problem, they approximate $\nabla_{\bar{\bx}_{\sigma_t}}L(\bar{\bx}_{\sigma_t},\bc)$ using $\nabla_{\bar{\bx}_{\sigma_t}}L(\hat{\bx}_{0}(\bar{\bx}_{\sigma_t},\sigma_t),\bc)$, where $\hat{\bx}_{0}(\bar{\bx}_{\sigma_t},\sigma_t)$ is an estimate of the clean image shown in~\eqref{eqn:estx0}. Besides using the above gradients as guidance, another technique called backward universal guidance is introduced in UG, which is used to enforce the generated image to satisfy the guidance. In this work, we do not use this technique as our method could already obtain high-quality generated results. 
%In DOODL~\cite{wallace2023endtoend}, 
% However, this one-step estimate in~\eqref{eqn:estx0} is usually misaligned with the final generated clean image, especially in the early stage of the sampling process where noised samples $\bar{\bx}_{\sigma}$ are far from the final outputs. Thus, this misalignment will cause the guidance inaccurate.  

Besides, directly adding guidance functions to standard generation pipelines may cause artifacts and deviations from the conditional controls. To mitigate this problem, the time-travel strategy in FreeDOM (or self-recurrence in UG) is applied. Specifically, after $\bx_{t-1}$ is sampled, we further add random Gaussian noise to $\bx_{t-1}$ and repeat this denoising and noising process for $r$ times before moving to the next sampling step.

\vspace{-0.3em}
\subsection{Adjoint Sensitivity Method}
\label{subsec:adjoint}
\vspace{-0.3em}
The outputs of neural ODE models involve multiple iterations of the function call, which introduces the challenge of backpropagating gradients to inputs and model weights because the vanilla gradient-backpropagation requires storing all the intermediate states and leads to extremely large memory consumption.
To solve this, \citet{chen2018neural} proposed the adjoint sensitivity method, in which adjoint states, $\ba_t = \frac{\partial L}{\partial \bx_t}$, are introduced to represent the gradients of the loss with respect to intermediate states. The adjoint method defines an augmented state as the pair of the system state $\bx_t$ and the adjoint variable $\ba_t$, and integrates the augmented state backward in time. 
% The adjoint variable represents the gradient $\frac{\partial L}{\partial \bx_t}$, and 
The backward integration of $\ba_t$ works as the gradient backpropagation in continuous time.

To obtain the gradients of loss w.r.t intermediate states in diffusion models, there are some existing works. DOODL~\cite{wallace2023endtoend} obtain the gradients of loss w.r.t noise vectors by using invertible neural networks~\cite{ardizzone2018analyzing}. DOODL relies on the invertibility of EDICT~\cite{wallace2023edict}, resulting in identical computation steps for both the backward gradient calculation and the forward sampling process. Here in our work, the $n$-step estimate could be flexible in the choice of $n$. FlowGrad~\cite{Liu_2023_CVPR} efficiently backpropagates the output to any intermediate time steps on the ODE trajectory, by decomposing the backpropagation and computing vector Jacobian products. FlowGrad needs to store the intermediate results, which may not be memory efficient. Moreover, the adjoint sensitivity method has been applied in diffusion models to finetune diffusion parameters for customization~\cite{pan2023adjointdpm} as it aids in obtaining the gradients of different parameters with efficient memory consumption. Specifically, in diffusion models (here we use the ODE form~\eqref{eqn:ode1}), we  first solve ODE~\eqref{eqn:ode1} from $T$ to $0$ to generate images, and then solve another backward ODE from $0$ to $T$ to obtain the gradients with respect to any intermediate state $\bar{\bx}_{\sigma_t}$, 
% and then $\frac{\partial L}{\partial \bx_t} = \frac{\partial L}{\partial \bar{\bx}_{\sigma_t}} \cdot \frac{1}{\sqrt{\alpha_t}}$. 
where the backward ODE~\cite{chen2018neural} has the following form:
\begin{equation}
\label{eqn:symp}
\setlength\abovedisplayskip{6pt}
\setlength\belowdisplayskip{6pt}
    \rmd \left[\begin{matrix}
        \bar{\bx}_{\sigma_t}\\
        \frac{\partial L}{\partial \bar{\bx}_{\sigma_t}}
    \end{matrix}\right] = \left[\begin{matrix}
        \bar{\epsilon} (\bar{\bx}_{\sigma_t},\sigma_t)\\
        -\left(\frac{\partial \bar{\epsilon} (\bar{\bx}_{\sigma_t},\sigma_t)}{\partial \bar{\bx}_{\sigma_t}} \right)^T\frac{\partial L}{\partial \bar{\bx}_{\sigma_t}}
    \end{matrix}\right] {\rmd \sigma_t}.
\end{equation}
After obtaining the gradients $\frac{\partial L}{\partial \bar{\bx}_{\sigma_t}}$, we naturally have $\frac{\partial L}{\partial \bx_t}=\frac{1}{\sqrt{\alpha_t}} \frac{\partial L}{\partial \bar{\bx}_{\sigma_t}}$ based on the definition of $\bar{\bx}_{\sigma_t}$.
% Combining adjoint methods into diffusion models aids in obtaining the gradients of different parameters with efficient memory consumption.  \citet{pan2023adjointdpm} also apply the adjoint method to diffusion models to finetune diffusion parameters for customization. 
Different from~\cite{pan2023adjointdpm}, this paper mainly focuses on flexible training-free guidance exploiting various types of information from pre-trained models. However, the vanilla adjoint method suffers from numerical errors. To reduce the numerical errors, backward integration often requires a smaller step size and leads to high computational costs.  In this work, we utilize the symplectic adjoint method~\cite{matsubara2021symplectic} to reduce the discretization error in solving the backward ODE.

% The adjoint method suffers from numerical errors and to suppress the numerical errors, the backward integration often requires a small step size. However, in our work, as we usually predict the $\hat{\bx}_0$ in a few steps (usually 4-5 steps), directly applying the adjoint method is not suitable. To obtain the exact gradients in a few sampling steps, we consider the symplectic adjoint method~\cite{matsubara2021symplectic, sanz2016symplectic}. 

% \begin{table}
%   \centering
%   \begin{tabular}{@{}lc@{}}
%     \toprule
%     Method & Frobnability \\
%     \midrule
%     Theirs & Frumpy \\
%     Yours & Frobbly \\
%     Ours & Makes one's heart Frob\\
%     \bottomrule
%   \end{tabular}
%   \caption{Results.   Ours is better.}
%   \label{tab:example}
% \end{table}

% When placing figures in \LaTeX, it's almost always best to use \verb+\includegraphics+, and to specify the figure width as a multiple of the line width as in the example below
% {\small\begin{verbatim}
%    \usepackage{graphicx} ...
%    \includegraphics[width=0.8\linewidth]
%                    {myfile.pdf}
% \end{verbatim}
% }

\section{Methods}
\label{sec:meth}

Existing training-free guidance methods usually construct the gradient guidance through a one-step estimate of the clean image $\hat{\bx}_0$, which, however, is usually misaligned with the finally generated clean image. Such misalignment worsens in the early stage of the sampling process where noised samples $\bar{\bx}_{\sigma_t}$ are far from the final outputs. As a result, the guidance is often inaccurate. 
To mitigate the misalignment of the one-step estimate of $\hat{\bx}_0$, we propose Symplectic Adjoint Guidance (SAG) to accurately estimate the finally generated content and provide exact gradient-based guidance for better generation quality. We first consider estimating the clean image $\hat{\bx}_0$  using $n$ steps from $\bx_t$. Nevertheless, when estimating $\hat{\bx}_0$ using $n$ steps from $\bx_t$, how to accurately obtain the gradients $\nabla_{\bar{\bx}_{\sigma_t}}L(\hat{\bx}_{0},\bc)$ is non-trivial. We utilize the symplectic adjoint method, which can obtain accurate gradients with efficient memory consumption. The overall pipeline is shown in Fig.~\ref{fig:intro} and the explicit algorithm is shown in Algorithm~\ref{alg:sag}.

%(Difference with UG).
\subsection{Multiple-step Estimation of Clean Outputs}
\label{subsec:nstep}

As discussed in section~\ref{subsec:tfree}, training-free guidance methods usually approximate $\nabla_{\bar{\bx}_{\sigma_t}}L(\bar{\bx}_{\sigma_t},\bc)$ using $\nabla_{\bar{\bx}_{\sigma_t}}L(\hat{\bx}_{0}(\bar{\bx}_{\sigma_t},\sigma_t),\bc)$. According to~\cite[Theorem 1]{chung2022diffusion}, the approximation error is upper bounded by two terms: the first is related to the norm of gradients, and the second is related to the average estimation error of the clean image, i.e., $m=\int \|\bx_0 - \hat{\bx}_0\|p(\bx_0|\bx_t)\, \rmd \bx_0$. To reduce the gradient estimation error, we consider reducing the estimation error of the clean image (i.e., the misalignment between the one-step estimate and the final generated clean image) by using the $n$ step estimate.

Suppose the standard sampling process generates clean outputs for $T$ steps, from which we sample a subset of steps for implementing guidance. The subset for guidance can be indicated via a sequence of boolean values, $\bg^{T:1}=[g_T, g_{T-1}, \cdots, g_1]$. For a certain step $t$ of guidance, we consider predicting the clean image by solving the ODE functions~\eqref{eqn:ode1} in $n$ time steps. Here we usually set $n$ to be much smaller than $T$ for time efficiency (refer to section~\ref{sec:ablation} for details). Here, note that we denote the state of the sub-process for predicting clean outputs as {\color{cvprblue}\textbf{$\bx'_t$ and $\bar{\bx}'_{\sigma}$}} so as to distinguish from the notation {\color{cvprblue}\textbf{$\bx_t$ and $\bar{\bx}_{\sigma}$}} used in the main sampling process. 
Taking solving~\eqref{eqn:ode1} using the Euler numerical solver~\cite{epperson2021introduction} as an example (when the standard generation process is at step $t$), the estimate of the clean image $\bx'_{0}$ can be solved iteratively by \eqref{eqn:forward}, where $\tau=n,\ldots,1$ and the initial state of this sub-process $\bx'_{n}=\bx_t$. 
\begin{equation}
\label{eqn:forward}
\setlength\abovedisplayskip{6pt}
\setlength\belowdisplayskip{6pt}
    \!\frac{\bx'_{\tau\!-\!1}}{\sqrt{\alpha_{\tau\!-\!1}}} \!=\! \frac{\bx'_{\tau}}{\sqrt{\alpha_{\tau}}}\! +\! \epsilon_{\theta}(\bx'_{\tau}, \tau) \left(\!\sqrt{\frac{1\!-\!\alpha_{\tau\!-\!1}}{\alpha_{\tau\!-\!1}}}\!-\!\sqrt{\frac{1\!-\!\alpha_{\tau}}{\alpha_{\tau}}}\!\right)
\end{equation}

For a special case that $n=1$ and $\sqrt{\alpha_0}=1$~\cite{song_denoising_2022}, we have
$
   \bx'_{0} = \frac{\bx_{t}}{\sqrt{\alpha_{t}}} - \epsilon_{\theta}(\bx_{t}, t) \sqrt{\frac{1-\alpha_{t}}{\alpha_{t}}}
$,
which is equivalent to~\eqref{eqn:estx0}. Thus,  our method particularizes to FreeDOM~\cite{Yu2023freedom} when $n=1$. Denote $m(n)$ to be the average estimation error in $n$ estimate steps.  In the following lemma, we show that $m$ will not increase when we use $n$-step estimation.
% This also can be reflected in Figure~\ref{fig:comp}. 
The proof of Lemma~\ref{lem:est} is shown in the Appendix~\ref{subapp:lem}.
\begin{lemma}
\label{lem:est}
    $m(n_1)\leq m(n_2)$ when $n_2\leq n_1$. 
\end{lemma}

\subsection{Symplectic Adjoint Method}
\label{subsec:sam}
% \HS{
% \begin{itemize}
%     \item vanilla numerical method for solve backward adjoint ODE
%     \item why it has discretization error?
%     \item we propose to use symplectic method for saving the backward adjoint ODE
%     \item a theorem/proof why symplectic method can reduce the error
% \end{itemize}
% }
% % (Difference with Adjoint)
\begin{figure}
    \centering
    \includegraphics[width=0.8\linewidth]{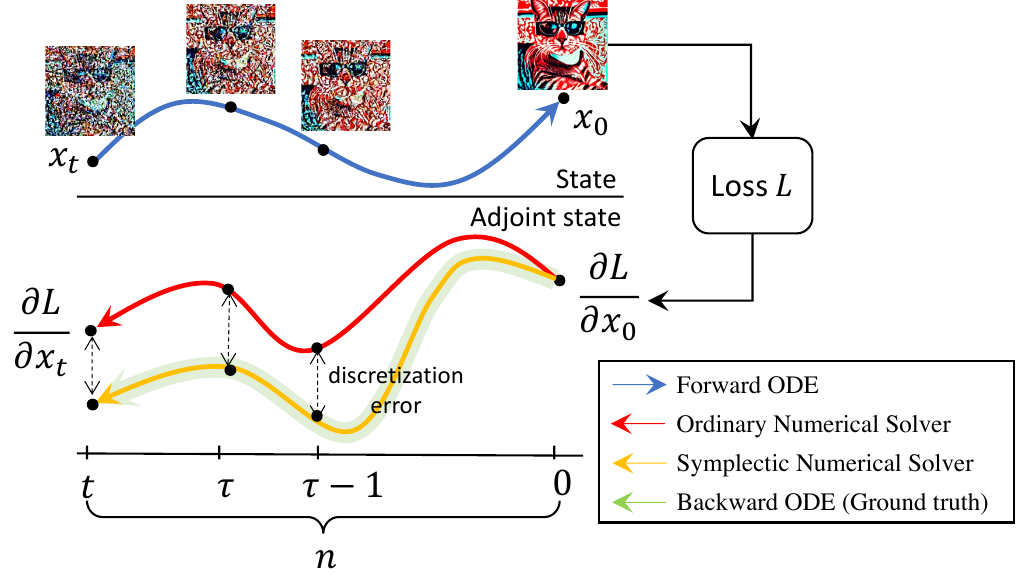}
    \vspace{-0.5em}
    \caption{Illustration of the Symplectic Adjoint method}
    \label{fig:comp}
    \vspace{-1em}
\end{figure}

In section~\ref{subsec:nstep}, we show how to get a more accurate estimation $\bx'_{0}$ of the final output by solving ODE functions~\eqref{eqn:ode1} in $n$ steps. As introduced in section~\ref{subsec:adjoint}, the adjoint method is a memory-efficient way to obtain the gradients $\frac{\partial L}{\partial \bx_t}$ through solving a backward ODE~\eqref{eqn:symp}. However, as our $n$ is set to be much smaller than $T$ and we usually set it to be 4 or 5 in our experiments, using the vanilla adjoint method will suffer from discretization errors. Thus, instead of using the vanilla adjoint method, we consider obtaining the accurate gradient guidance $\nabla_{\bx_t}L(\bx'_{0},\bc)$ using Symplectic Adjoint method~\cite{feng2010symplectic,matsubara2021symplectic}.
Here we present the first-order symplectic Euler solver~\cite{feng2010symplectic} as an example to solve~\eqref{eqn:symp} from $0$ to $n$ to obtain accurate gradients. We also can extend it to high-order symplectic solvers, such as  Symplectic Runge–Kutta Method~\cite{matsubara2021symplectic} for further efficiency in solving (refer to Appendix~\ref{subapp:highorder}).

Suppose we are implementing guidance at time step $t$, the forward estimation sub-process is discretized into $n$ steps. Let $\tau \in[n,\ldots,0]$ denote the discrete steps corresponding to time from $t$ to $0$ and $\sigma_{\tau}=\sqrt{1-\alpha_{\tau}}/\sqrt{\alpha_{\tau}}$. 
The forward estimate follows the forward update rule~\eqref{eqn:forward}, whose continuous form equals ODE~\eqref{eqn:ode1}.
Then, the Symplectic Euler update rule for solving the corresponding backward ODE~\eqref{eqn:symp} is:
\begin{align}
    \bar{\bx}'_{\sigma_{\tau+1}}& 
    = \bar{\bx}'_{\sigma_{\tau}} 
    + h_{\sigma_{\tau}} \bar{\epsilon}(\bar{\bx}'_{\sigma_{\tau+1}}, \sigma_{\tau+1})\label{eqn:sympup1},\\
    \frac{\partial L } {\partial \bar{\bx}'_{\sigma_{\tau+1}}} & \!
    = \frac{\partial L } {\partial \bar{\bx}'_{\sigma_{\tau}}}\! 
    - \!h_{\sigma_{\tau}}\left(\frac{\partial \bar{\epsilon} (\bar{\bx}'_{\sigma_{\tau+1}}, \sigma_{\tau+1})}{\partial \bar{\bx}'} \right)^T \frac{\partial L } {\partial \bar{\bx}'_{\sigma_{\tau}}}\label{eqn:sympup2},
\end{align}
for $\tau=0, 1, \ldots, n-1$. $h_{\sigma}$ is the discretization step size. After we obtain $\frac{\partial L } {\partial \bar{\bx}'_{\sigma_{n}}}$, $\frac{\partial L } {\partial \bx_t}$ is easily computed by $\frac{\partial L}{\partial \bar{\bx}'_{\sigma_n}}\cdot \frac{1}{\sqrt{\alpha_t}}$ based on the definition of $\bar{\bx}_{\sigma_t}$.

Note that \textit{different from the vanilla adjoint sensitivity method}, which uses $\bar{\epsilon}({\color{cvprblue}\bar{\bx}'_{\sigma_{\tau}}}, \sigma_{\tau})$ to update $\bar{\bx}'_{\sigma_{\tau+1}}$ and $\frac{\partial L } {\partial \bar{\bx}'_{\sigma_{\tau+1}}}$, the proposed symplectic solver uses $\bar{\epsilon}({\color{cvprblue}\bar{\bx}'_{\sigma_{\tau+1}}}, \sigma_{\tau+1})$. The values of $\bar{\bx}'_{\sigma_{\tau+1}}$ are restored from those that have been computed during the forward estimation. 
% This is the meaning of \emph{``symplectic''} as we align the numerical solutions of forward and backward processes. 
% Without symplectic technique, numerical solutions will mismatch and cause errors in obtain the gradients. 
In Theorem~\ref{thm:symp}, we prove that the gradients obtained by the Symplectic Euler are accurate. 
% We assume that the starting point of the $n$-step estimate is $\bx_t$. 
Due to the limits of space, the complete statement and proof of Theorem~\ref{thm:symp} are presented in  Appendix~\ref{subapp:them}.  We illustrate the difference between the vanilla adjoint method and the symplectic adjoint method 
% and the difference between our $n$-step estimate and one-step estimate used in previous works~\cite{Yu2023freedom,Bansal_2023_CVPR} 
in Fig.~\ref{fig:comp}.
\vspace{-1em}
\begin{theorem} (Informal)
\label{thm:symp}
    Let the gradient $\frac{\partial L}{\partial \bar{\bx}'_{\sigma_t}}$  be the analytical solution to the continuous ODE in~\eqref{eqn:symp} 
    and let $\frac{\partial L}{\partial \bar{\bx}'_{\sigma_{n}}}$ be the gradient obtained by the symplectic Euler solver in ~\eqref{eqn:sympup2} throughout the discrete sampling process. Then, under some regularity conditions, we have 
    $\frac{\partial L}{\partial \bar{\bx}'_{\sigma_t}} = \frac{\partial L}{\partial \bar{\bx}'_{\sigma_{n}}}$.
\end{theorem}

\begin{algorithm}[t!]
\caption{Symplectic Adjoint Guidance (SAG)}
\label{alg:sag}
\begin{algorithmic}[1]
\REQUIRE diffusion model $\epsilon_{\theta}$, condition $\bc$, loss $L$, sampling scheduler $\calS$,  guidance strengths $\rho_{t}$, noise scheduling $\alpha_t$, guidance indicator $[g_T,\ldots,g_1]$, repeat times of time travel $(r_T, \ldots,r_1)$.
\STATE $\bx_T\sim\calN(0, \bI)$
\FOR{$t=T, \ldots, 1$}
\FOR{$i=r_t, \ldots, 1$}
\STATE $\bx_{t-1} = \calS(\bx_t, \epsilon_{\theta}, \bc)$
\IF{$g_t$}
\STATE $\hat{\bx}_0=$ solving~\eqref{eqn:forward} 
in $n$ steps
\STATE $\nabla_{\bx_t}L(\hat{\bx}_{0},\bc)$ =  solving~\eqref{eqn:sympup1} and~\eqref{eqn:sympup2}.
\STATE $\bx_{t-1} = \bx_{t-1} - \rho_{t}\nabla_{\bx_t}L(\hat{\bx}_{0},\bc) $
\ENDIF
% \STATE $\epsilon'\sim \calN(0, \bI)$
\STATE $\bx_t=\frac{\sqrt{\alpha_t}}{\sqrt{\alpha_{t-1}}}\bx_{t-1} + \frac{\sqrt{\alpha_{t-1}-\alpha_t}}{\sqrt{\alpha_{t-1}}}\epsilon'$  with $\epsilon'\sim \calN(0, \bI)$
\ENDFOR
\ENDFOR
\end{algorithmic}
\end{algorithm}

% \vspace{-0.5em}
Combining the two stages above, namely the $n$-step estimate of the clean output and the symplectic adjoint method, we have the Symplectic Adjoint Guidance (SAG) method, which is shown in Algorithm~\ref{alg:sag}. We also apply the time-travel strategy~\cite{Bansal_2023_CVPR,Yu2023freedom} into our algorithm. The sampling/denoising scheduler $\calS$ could be any popular sampling algorithm, including DDIM~\cite{song_denoising_2022}, DPM-solver~\cite{lu2022dpm}, and DEIS~\cite{zhang2022fast}. The overall illustration of our SAG is shown in Fig.~\ref{fig:intro}. 

\vspace{-0.5em}
\paragraph{Runtime analysis}
While increasing $n$ can mitigate the misalignment of $\hat{\bx}'_0$ and lead to a highly enhanced quality of generated images in different tasks, it also proportionally increases the runtime. 
% We consider the computation cost of one-time function evaluation $\epsilon_{\theta}(\cdot)$ and one-time gradient calculation $\nicefrac{\partial \epsilon_{\theta}}{\partial \bx'}$ as one unit of function/network evaluation.
% For the $n$-step symplectic adjoint guidance, the extra \textbf{n}umber of \textbf{f}unction/network \textbf{e}valuations (NFE) becomes $2n$ which include all the function evaluations during estimation and gradient backpropagation. 
There exists a trade-off between computational cost and the generation quality. When $n=1$, SAG degenerates to one-step-estimate guidance methods (\eg FreeDOM~\cite{Yu2023freedom}). The computation cost decreases but the sample quality is compromised. In practice, we 
% Thus, our method will be slower than existing training-free methods and this is the main limitation of our method. 
could design an adaptive guidance strategy where the number of estimate steps dynamically adjusts itself as the main sampling process proceeds. For example,  we may use a relatively large $n$ at the early sampling stage and then gradually decrease to the one-step estimate when $\hat{\bx}'_0$ is not far from the final generations.
Besides adaptively adjusting the number of estimate steps $n$, SAG also allows us to select the subset of intermediate steps of the main sampling process for guiding, which is indicated by $\bg^{T:1}$. Usually, we only choose a sequence of the middle stage for guiding, \ie, $g_{t}=1$ for $t\in[K_2, K_1]$ with $0<K_1<K_2<T$ and $g_{t}=0$ for others. That is because the states at the very early denoising stage are less informative about the final outputs, and the states at the very last stage almost decide the appearance of final outputs and barely can make more changes.
\vspace{-0.5em}
\section{Experiments}
\label{sec:exp}
\vspace{-0.5em}
We present experimental results to show the effectiveness of SAG. We apply SAG to several image and video generation tasks, including style-guided image generation, image aesthetic improvement, personalized image generation (object guidance and face-ID guidance), and video stylization. We conduct ablation experiments to study the effectiveness of hyperparameters including the number of estimation steps $n$, guidance scale $\rho_t$, \textit{etc}.

\begin{figure}
    \centering
    \includegraphics[width=0.9\linewidth]{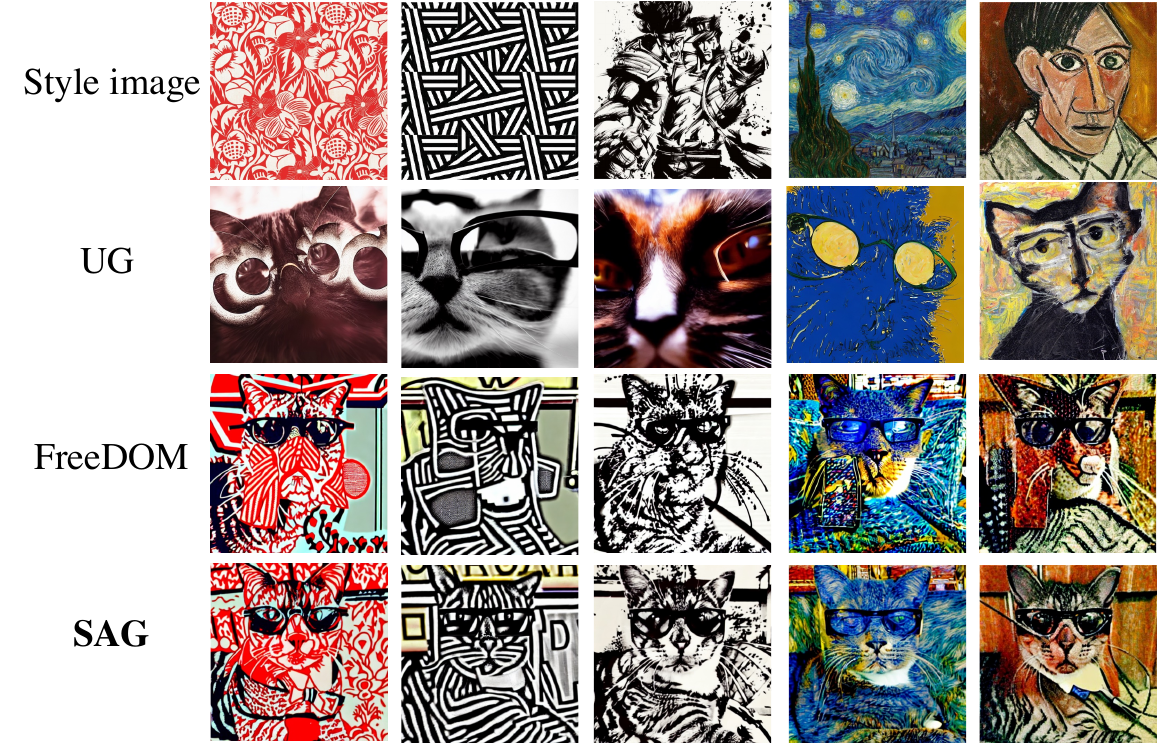}
    \caption{Stylization results of \emph{``A cat wearing glasses"}.}
    \label{fig:style}
\end{figure}

\begin{table}
\begin{subtable}[h]{0.48\linewidth}
  \centering
  \scalebox{0.7}{\begin{tabular}{@{}lcc@{}}
    \toprule
    Method & Style loss ($\downarrow$) & CLIP ($\uparrow$) \\
    \midrule
    FreeDOM & 482.7 & 22.37\\
    UG & 805 & 23.02\\
    SAG & \textbf{386.6} & \textbf{23.51}\\
    \bottomrule
  \end{tabular}}
  \caption{Style guided generation.}
  \label{table:style}
\end{subtable}
\begin{subtable}[h]{0.42\linewidth}
  \centering
  \scalebox{0.7}{\begin{tabular}{@{}lcc@{}}
    \toprule
    Method & ID loss ($\downarrow$) & FID ($\downarrow$) \\
    \midrule
    FreeDOM & 0.602 & 65.24\\
    SAG & \textbf{0.574} & \textbf{64.25}\\
    ~ \\
    \bottomrule
  \end{tabular}}
  \caption{Face-ID guided generation.}
  \label{table:face}
\end{subtable}
\caption{Quantitative Comparison: (a) Stylization quality measured by style loss and clip Score, (b) Performance of face ID guided generation assessed using face ID Loss and FID.}
\end{table}

\vspace{-0.5em}
\subsection{Style-Guided Sampling}
\label{subsec:style}
\vspace{-0.5em}
% \JH{Add one sentence to briefly describe style-guided sampling setup (\eg, input, output, goal).}
Style-guided sampling generates an output image that seamlessly merges the content's structure with the chosen stylistic elements from reference style images, showcasing a harmonious blend of content and style. To perform style-guided sampling, following the implementation of~\cite{Yu2023freedom}, we use the features from the third layer of the CLIP image encoder as our feature vector. The loss function is $L_2$-norm between the Gram matrix of the style image and the Gram matrix of the estimated clean image. We use the gradients of this loss function to guide the generation in Stable Diffusion~\cite{rombach_high-resolution_2022}. 
We set $n=4$.
% Besides, in the stylization, we set $n=4$.

We compare our results with FreeDOM~\cite{Yu2023freedom} and Universal Guidance (UG)~\cite{Bansal_2023_CVPR}.  We use style loss as a metric to measure the stylization performance and use the CLIP~\cite{radford_learning_2021} score to measure the similarity between generated images and input prompts. Good stylization implies that the style of generated images should be close to the reference style image while aligning with the given prompt. We obtain the quantitative results by randomly selecting five style images and four prompts, generating five images per style and per prompt. We use the officially released codes to generate the results of FreeDOM\footnote{\url{https://github.com/vvictoryuki/FreeDoM.git}} and Universal Guidance\footnote{\url{https://github.com/arpitbansal297/Universal-Guided-Diffusion.git}} under the same style and prompt. The qualitative results are shown in Fig.~\ref{fig:style} and quantitative results are shown in Table~\ref{table:style}. Full details can be found in Appendix~\ref{subapp:style} and more results in Appendix~\ref{app:qual}.

From Fig.~\ref{fig:style} and Table.~\ref{table:style}, we can find that SAG has the best performance compared with FreeDOM and UG as it has better stylization phenomena and it can largely preserve the content of images with the given text prompt. Besides, it is obvious that UG performs the worst in terms of stylization. We can observe that stylization by UG is not obvious for some style images and the image content is distorted.

\begin{figure}[h!]
    \centering
    \includegraphics[width=0.8\linewidth]{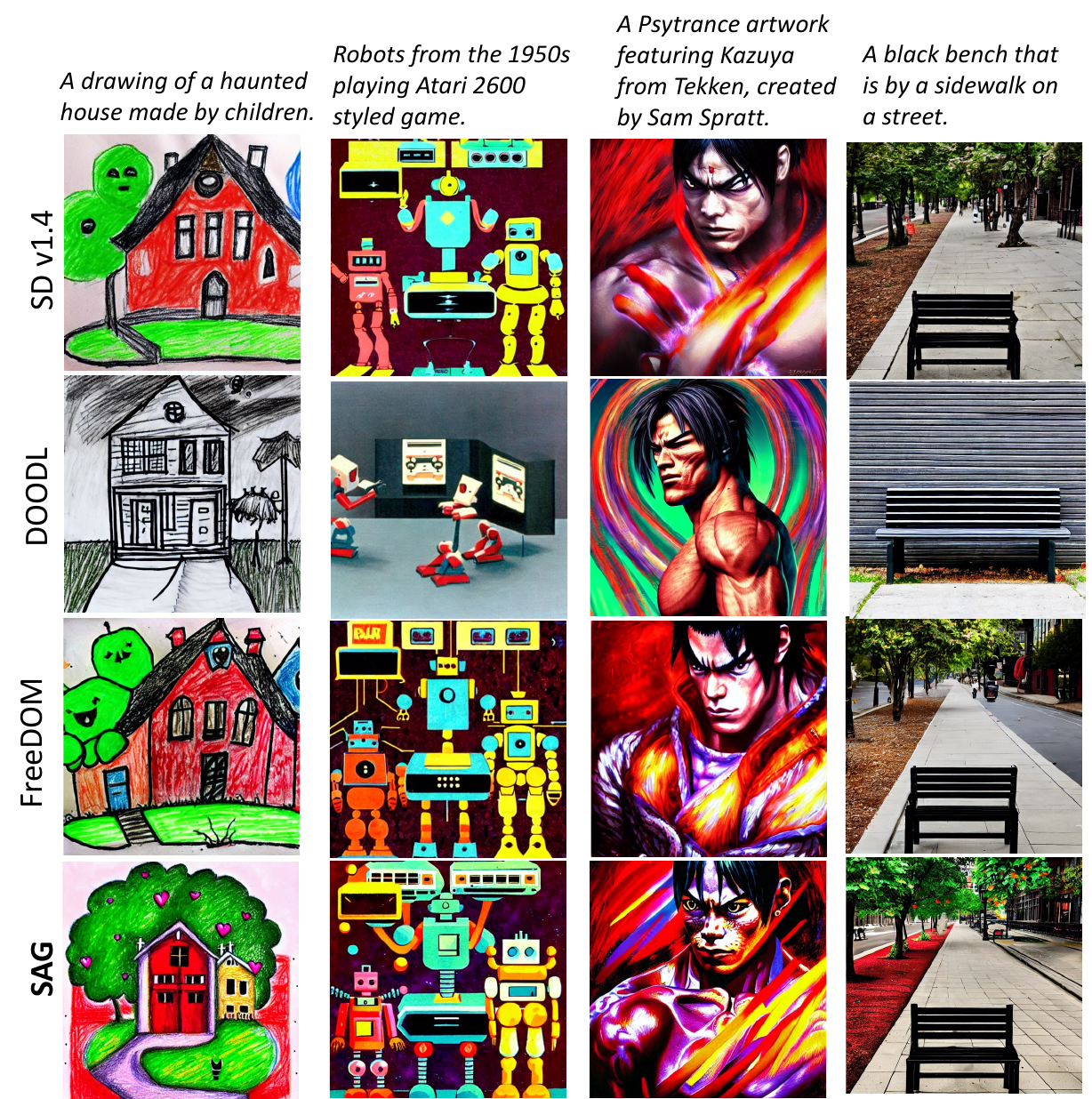}
    \caption{Examples on aesthetic improvement}
    \label{fig:aes}
    \vspace{-0.5em}
\end{figure}

\vspace{-0.5em}
\subsection{Aesthetic Improvement}
\label{subsec:aes}
\vspace{-0.5em}
% Besides using a reference image to guide the diffusion generation, we could also use some metrics as guidance information. 
In this task, we consider improving the aesthetic quality of generated images through the guidance of aesthetic scores obtained by the LAION aesthetic predictor,\footnote{\url{https://github.com/LAION-AI/aesthetic-predictor.git}} PickScore~\cite{kirstain2023pick} and HPSv2~\cite{wu2023human}. The LAION aesthetic predictor is a linear head pre-trained on top of CLIP visual embeddings to predict a value ranging from 1 to 10, which indicates the aesthetic quality. PickScore and HPSv2 are two reward functions trained on human preference data. We set $n=4$ and use the linear combination of these three scores as metrics to guide image generation. We randomly select ten prompts from four prompt categories, \emph{Animation, Concept Art, Paintings, Photos}, and generate one image for each prompt. We compare the resulting weighted aesthetic scores of all generated images with baseline Stable Diffusion (SD) v1.5, DOODL~\cite{wallace2023endtoend} and FreeDOM~\cite{Yu2023freedom} in Table~\ref{table:aes}. The results were generated using the official code released by DOODL.\footnote{\url{https://github.com/salesforce/DOODL.git}} The qualitative comparison is shown in Fig~\ref{fig:aes}. We find that our method has the best aesthetic improvement effect, with more details and richer color. Besides, as DOODL optimizes the initial noise to enhance aesthetics, the generated images will be different from the original generated images. Experimental details are shown in Appendix~\ref{subapp:aes} and more results in Appendix~\ref{app:qual}.

% \begin{table}
%   \centering
%   \scalebox{0.8}{\begin{tabular}{@{}lc@{}}
%     \toprule
%     Method & Aesthetic Score ($\downarrow$)  \\
%     \midrule
%     Stable Diffusion & 9.71 \\
%     FreeDOM & 9.18 \\
%     DOODL &  9.78 \\
%     SAG & \textbf{8.17} \\
%     \bottomrule
%   \end{tabular}}
%   \caption{Aesthetic score comparison.}
%   \label{table:aes}
% \end{table}

\begin{table}
\begin{subtable}[h]{0.45\linewidth}
  \centering
  \scalebox{0.7}{\begin{tabular}{@{}lc@{}}
    \toprule
    Method & Aesthetic loss($\downarrow$)  \\
    \midrule
    SD v1.5 & 9.71 \\
    FreeDOM & 9.18 \\
    DOODL &  9.78 \\
    SAG & \textbf{8.17} \\
    \bottomrule
  \end{tabular}}
  \caption{Aesthetic improvement.}
  \label{table:aes}
\end{subtable}
\begin{subtable}[h]{0.45\linewidth}
 \centering
  \scalebox{0.7}{\begin{tabular}{@{}lcc@{}}
    \toprule
    Method & CLIP-I ($\uparrow$) & CLIP-T ($\uparrow$) \\
    \midrule
    DreamBooth & 0.724 & 0.277\\
    FreeDOM & 0.681 & \textbf{0.281}\\
    DOODL &  0.743 & 0.277\\
    SAG & \textbf{0.774} & 0.270\\
    \bottomrule
  \end{tabular}}
  \caption{Object guided generation.}
  \label{table:person}
\end{subtable}
\caption{Quantitative Comparison: (a) Aesthetic loss for image aesthetics, (b) Clip image and clip text scores for object-guided generation performance.}
\vspace{-1em}
\end{table}

\subsection{Personalization}
\vspace{-0.5em}
\label{subsec:person}

Personalization aims to generate images that contain a highly specific subject in new contexts.
% such that they are naturally combined into the scene. 
When given a few images (usually 3-5) of a specific subject, DreamBooth~\cite{ruiz_dreambooth_2022} and Textual Inversion~\cite{gal2023an} learn or finetune components such as text embedding or subsets of diffusion model parameters to blend the subject into generated images. However, when there is only a single image, the performance is not satisfactory. In this section, we use symplectic adjoint guidance to generate a personalized image without additional generative model training or tuning based on a single example image. We conduct experiments with two settings: (1) general object guidance and (2) face-ID guidance.
\begin{figure}[h!]
    \vspace{-.5em}
    \centering
    \includegraphics[width=0.7\linewidth]{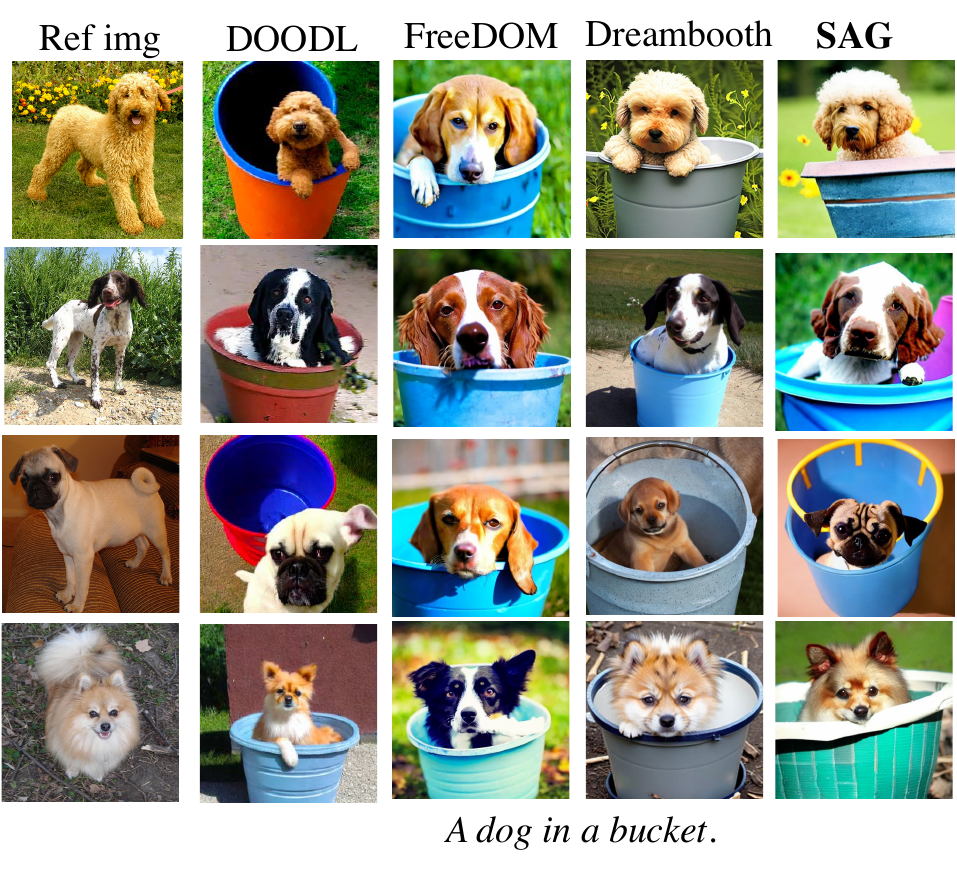}
    \vspace{-1em}
    \caption{Examples on object-guided sampling}
    \label{fig:person}
\end{figure}

\vspace{-1.5em}
\paragraph{Object Guidance}
We first do the personalization of certain objects in Stable Diffusion. We use a spherical distance loss~\cite{wallace2023endtoend} to compute the distance between image features of generated images and reference images obtained from ViT-H-14 CLIP model.\footnote{\url{https://github.com/mlfoundations/open_clip.git}} In this task, we set $n=4$. We compare our results with FreeDOM~\cite{Yu2023freedom}, DOODL~\cite{wallace2023endtoend} and DreamBooth~\cite{ruiz_dreambooth_2022}. The results of DreamBooth\footnote{\url{https://github.com/XavierXiao/Dreambooth-Stable-Diffusion.git}} is generated using the official code. We use DreamBooth to finetune the model for 400 steps and set the learning rate as $1\times 10^{-6}$ with only one training sample. We use the cosine similarity between CLIP~\cite{radford_learning_2021} embeddings of generated and reference images (denoted as CLIP-I) and the cosine similarity between CLIP embeddings of generated images and given text prompts (denoted as CLIP-T) to measure the performance. The quantitative comparison is shown in Table.~\ref{table:person} and the qualitative results are shown in Fig.~\ref{fig:person}. We can find that images generated by SAG have the highest CLIP image similarity with reference images. We show experimental details in Appendix~\ref{subapp:person} and more results in Appendix~\ref{app:qual}.
\begin{figure}[h!]
\vspace{-0.5em}
    \centering
    \includegraphics[width=\linewidth]{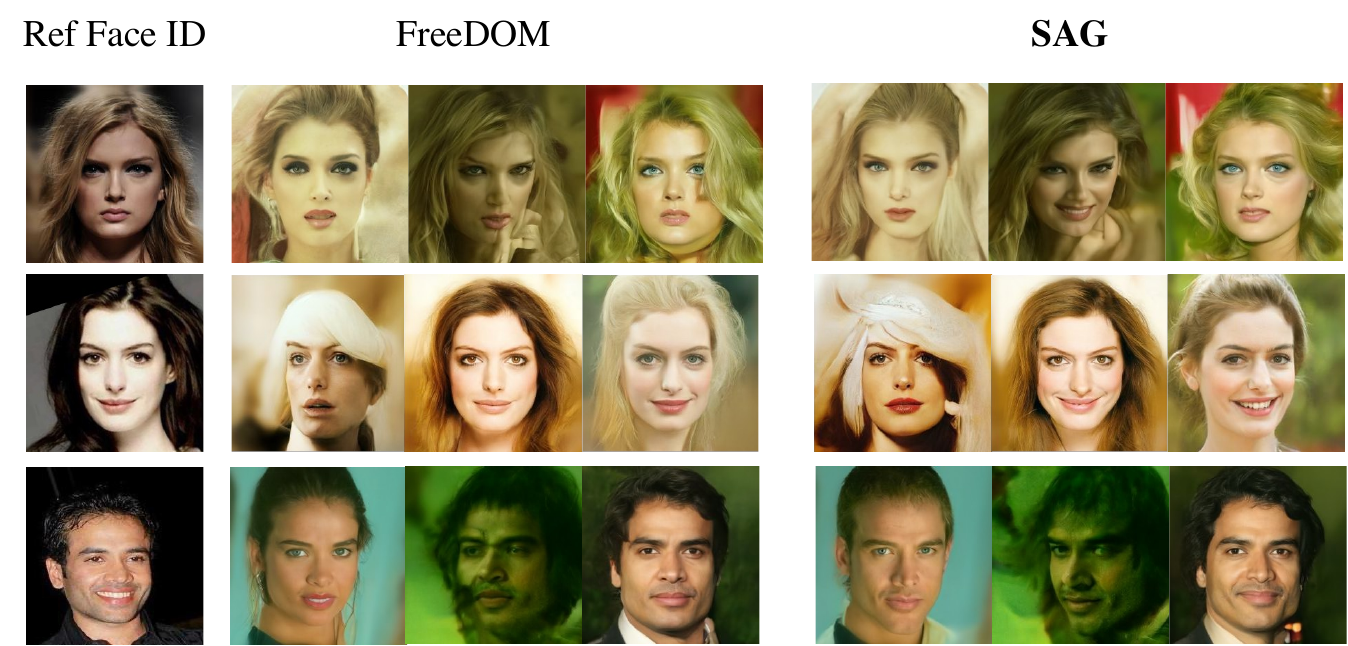}
    \caption{Examples on Face ID guided generation}
    \label{fig:face}
\end{figure}

\vspace{-1.5em}
\paragraph{Face-ID Guidance}
Following the implementation of~\cite{Yu2023freedom}, we use ArcFace to extract the target features of reference faces to represent face IDs and compute the $l_2$ Euclidean distance between the extracted ID features of the estimated clean image and the reference face image as the loss function. In this task, we set $n=5$. 
We compare our Face ID guided generation results with FreeDOM and measure the performance using the loss and FID, respectively.  We randomly select five face IDs and generate 200 faces for each face IDs. We show the qualitative results in Fig.~\ref{fig:face} and the quantitative results in Table~\ref{table:face}. Compared with FreeDOM, SAG matches the conditional Face IDs better with better generation image quality (lower FID). 

% \begin{table}
%   \centering
%   \scalebox{0.8}{\begin{tabular}{@{}lcc@{}}
%     \toprule
%     Method & ID loss ($\downarrow$) & FID ($\downarrow$) \\
%     \midrule
%     FreeDOM & 0.602 & 65.24\\
%     SAG & \textbf{0.574} & \textbf{64.25}\\
%     \bottomrule
%   \end{tabular}}
%   \caption{FID and Face ID loss comparison.}
%   \label{table:face}
% \end{table}

\subsection{Video Stylization}
\label{subsec:exp_video}
\vspace{-0.5em}
We also apply the SAG method for style-guided video editing, where we change the content and style of the original video while keeping the motion unchanged. For example, given a video of a dog running, we want to generate a video of a cat running with a sketch painting style. In this experiment, we use MagicEdit~\cite{liew2023magicedit}, which is a video generation model conditioning on a text prompt and a sequence of depth maps for video editing. Given an input video, we first extract a sequence of depth maps. By conditioning on the depth maps, MagicEdit renders a video whose motion follows that in the original video. Using the style Gram metric in Sec. \ref{subsec:style}, we can compute the average loss between each frame and the reference style image. 

Since the depth and text conditions provide very rich information about the final output, MagicEdit can synthesize high-quality videos within 25 steps (\ie denoising for $T$ from $25$ to $0$).
% and the intermediate noisy states in the early steps will reflect the information on color style and layout. 
We use MagicEdit to render video of 16 frames where the resolution of each frame is $256\times 256$. We apply the SAG guidance to the steps of $T\in[20,10]$. As shown in Figure \ref{fig:exp-video}, SAG can effectively enable MagicEdit to generate videos of specific styles (\eg, a cat of the Chinese papercut style). In contrast, without SAG, the base editing model can barely synthesize videos whose color and texture align with the reference image. More experimental details are shown in Appendix~\ref{app:video}.

\begin{figure}
    \centering
    \includegraphics[width=\linewidth]{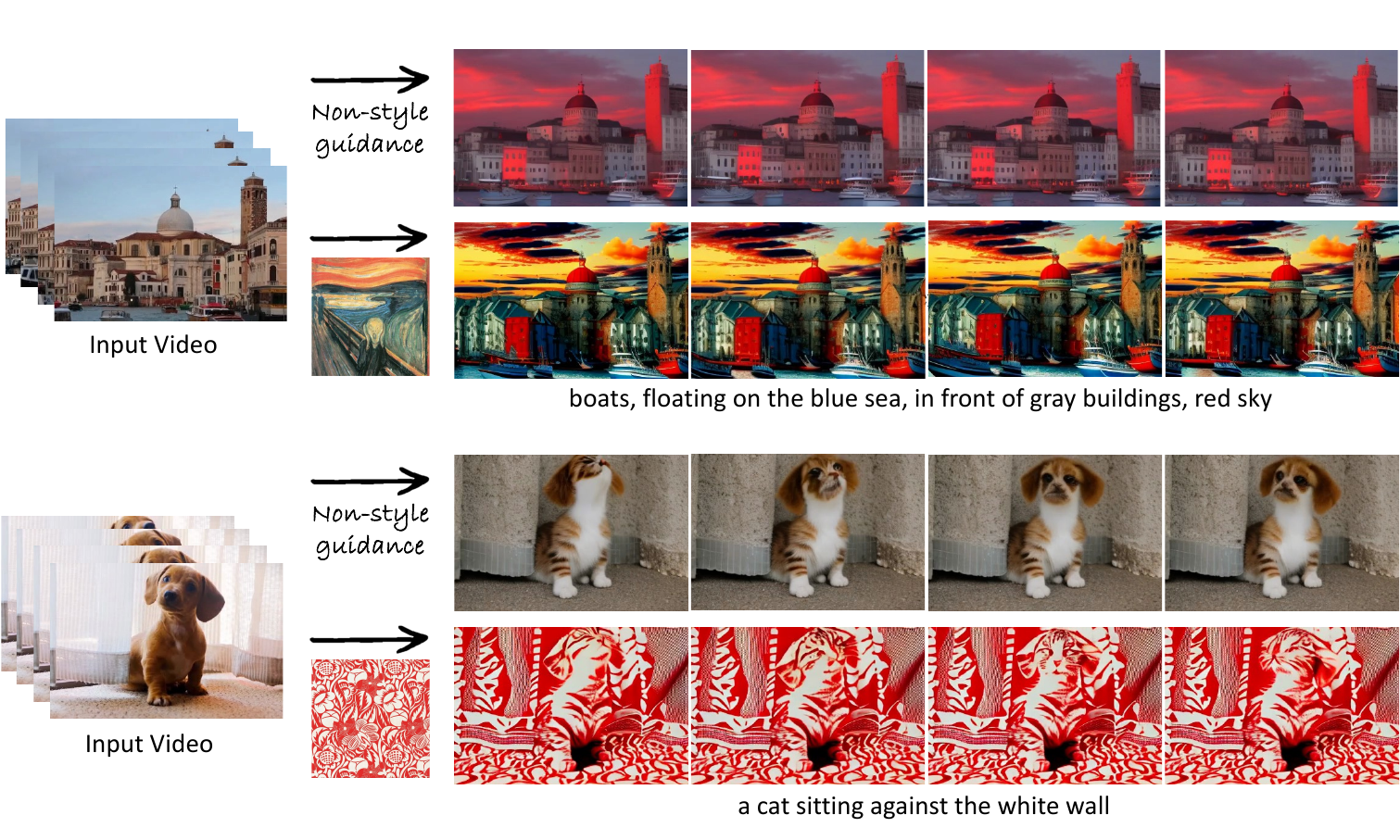}
    \caption{Examples on Video Stylization. For each input, the upper row is rendered on the conditioning of a text prompt and the depth sequence. The lower row is the output with the extra style guidance.}
    \label{fig:exp-video}
    \vspace{-0.5em}
\end{figure}

\subsection{Ablation Study}
\label{sec:ablation}
\vspace{-0.5em}

\paragraph{Choice of $n$.}
We investigate the impact of varying values of $n$ on the model's performance. Taking the stylization task as an example, we set $T=100$ and perform training-free guidance from step 70 to step 31. We use the prompts: \emph{``A cat wearing glasses"}, \emph{``butterfly"} and \emph{``A photo of an Eiffel Tower"} to generate 20 stylized images for each $n$.  The results are shown in Fig.~\ref{fig:abl_n} and the loss curve is in Fig.~\ref{fig:n_loss}. We can observe that when $n=1$ which reduces to  FreeDOM~\cite{Yu2023freedom}), the stylized images suffer from content distortion and less obvious stylization effect. As $n$ increases, both the quality of generated images and the reduction in loss between generated images and style images become more prominent. Notably, when $n$ increases beyond 4, there is no significant decrease in loss, indicating that setting $n$ to a large value is unnecessary. Besides, we notice that a small value of $n$, as long as greater than 1, could significantly help improve the quality of generated images. In most experiments, we set $n$ to be 4 or 5. 

% Note that while increasing $n$ can highly enhance the quality of generated images in different tasks, it also proportionally increases the runtime. For $n$-step symplectic adjoint guidance, the Neural Function Evaluations (NFEs) are doubled, resulting in $2n$ NFEs to derive the guidance function. In contrast, the one-step estimate requires only a single NFE to obtain the guidance function. Thus, our method will be slower than existing training-free methods. This is the main limitation of our method. Nevertheless, our approach offers a substantial improvement in the quality of training-free guided image generation.  

\begin{figure}
    \centering
    \includegraphics[width=0.7\linewidth]{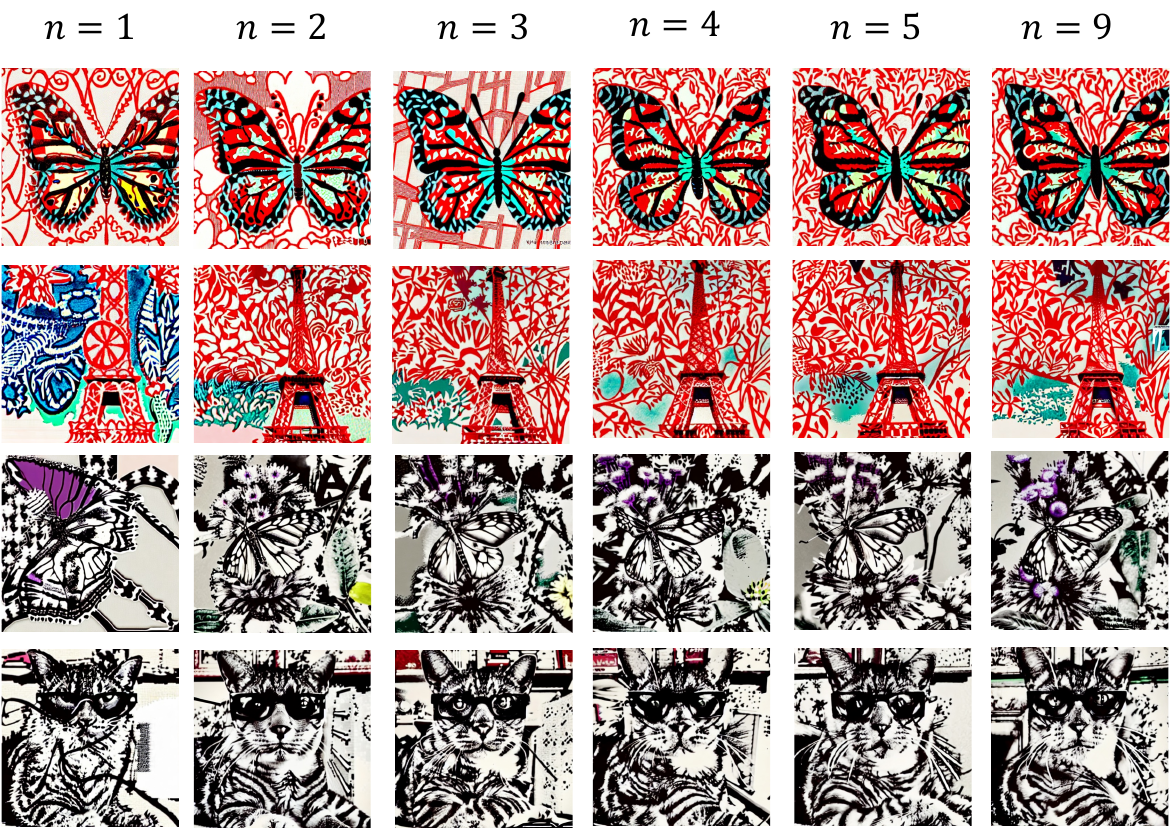}
    \caption{Stylization results with varying $n$.}
    \label{fig:abl_n}
    \vspace{-1em}
\end{figure}

\vspace{-0.5em}
\paragraph{Guidance scale $\rho_{t}$.}
We then study the influence of the guidance scale on the performance. Once again, we take stylization as an example and test the results under $n=1$ and $n=3$. We gradually increase the guidance scale and show the results in Fig.~\ref{fig:rho}. We can observe that when the scale increases, the stylization becomes more obvious, but when the scale gets too large, the generated images suffer from severe artifacts. 
% Besides, comparing the results of $n=1$ to that of $n=3$, increasing the scale cannot help to mitigate the distortion of content. 
\begin{figure}[h!]
    \centering
    \includegraphics[width=0.8\linewidth]{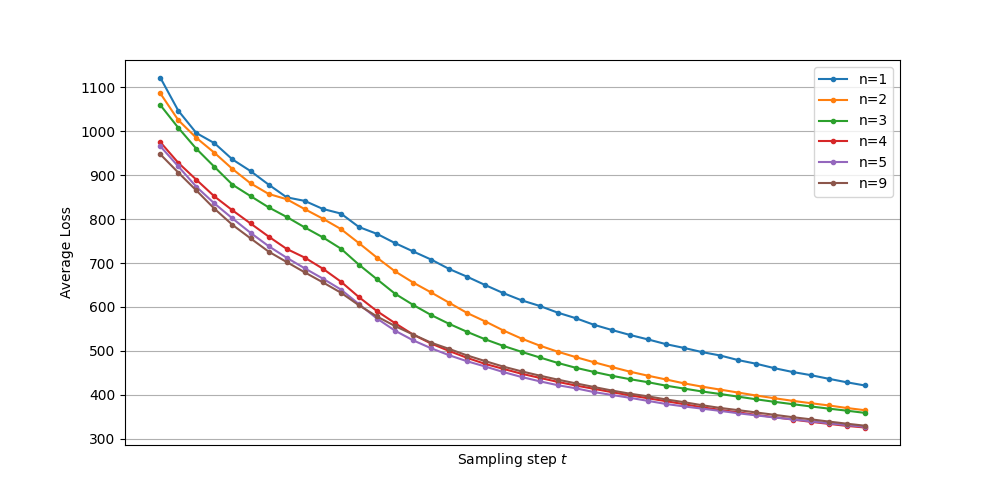}
    \caption{Loss curves for stylization under different $n$.}
    \label{fig:n_loss}
    \vspace{-0.5em}
\end{figure}

\begin{figure}[h!]
    \vspace{-0.6em}
    \centering
    \includegraphics[width=\linewidth]{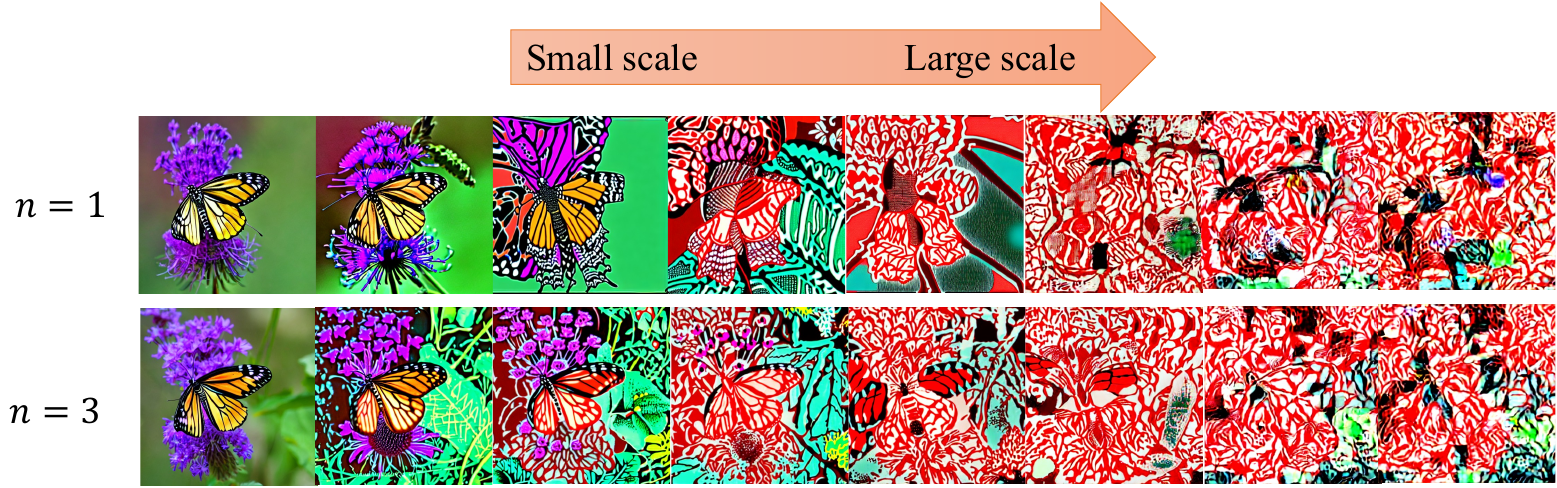}
    \caption{Stylization results when increasing scales.}
    \vspace{-0.6em}
    \label{fig:rho}
\end{figure}
\vspace{-1em}
\paragraph{Choice of guidance steps and repeat times of time travel.}
Finally, we also conduct experiments to study at which sampling steps should we do training-free guidance and the repeat times of time travel. As discussed in~\cite{Yu2023freedom}, the diffusion sampling process roughly includes three stages: the chaotic stage where $\bx_t$ is highly noisy, the semantic stage where $\bx_t$ presents some semantics and the refinement stage where changes in the generated results are minimal. Besides, for the repeat times, intuitively, increasing the repeat times extends the diffusion sampling process and helps to explore results that satisfy both guidance and image quality. Thus, in tasks such as stylization and aesthetic improvement that do not require change in content, we only need to do guidance in the semantic stage with a few repeat times (in these tasks, we set repeat times to 2). On the other hand, for tasks such as personalization, we need to perform guidance at the chaotic stage and use larger repeat times (here we set it to 3). More ablation study results are shown in Appendix~\ref{subapp:abla}.

{
    \small
    \bibliographystyle{ieeenat_fullname}
    \bibliography{main}
}
\clearpage
\setcounter{page}{1}
\maketitlesupplementary

% \section{Rationale}
% \label{sec:rationale}
% % 
% Having the supplementary compiled together with the main paper means that:
% % 
% \begin{itemize}
% \item The supplementary can back-reference sections of the main paper, for example, we can refer to \cref{sec:intro};
% \item The main paper can forward reference sub-sections within the supplementary explicitly (e.g. referring to a particular experiment); 
% \item When submitted to arXiv, the supplementary will already included at the end of the paper.
% \end{itemize}
% % 
% To split the supplementary pages from the main paper, you can use \href{https://support.apple.com/en-ca/guide/preview/prvw11793/mac#:~:text=Delete%20a%20page%20from%20a,or%20choose%20Edit%20%3E%20Delete).}{Preview (on macOS)}, \href{https://www.adobe.com/acrobat/how-to/delete-pages-from-pdf.html#:~:text=Choose%20%E2%80%9CTools%E2%80%9D%20%3E%20%E2%80%9COrganize,or%20pages%20from%20the%20file.}{Adobe Acrobat} (on all OSs), as well as \href{https://superuser.com/questions/517986/is-it-possible-to-delete-some-pages-of-a-pdf-document}{command line tools}.
\appendix
\section{Theoretical details on Symplectic Adjoint Guidance (SAG)}
\label{app:theory}

\subsection{Proof of Lemma~\ref{lem:est}}
\label{subapp:lem}

We know that the sub-process {\color{cvprblue}\textbf{$\bx'_{\tau}$ and $\bar{\bx}'_{\sigma}$}} also satisfy the following ODE with the initial condition of $\bar{\bx}_{\sigma_t}$ (i.e., $\bx_t$):
\begin{equation}
\label{eqn:appode}
\rmd \bar{\bx}'_{\sigma_{\tau}}= \bar{\epsilon} (\bar{\bx}'_{\sigma_{\tau}},\sigma_{\tau})\rmd \sigma_{\tau}.
\end{equation}
This means that when given the initial condition of $\bx_t$, the samples generated by solving the subprocess ODE~\eqref{eqn:appode} also share the same conditional probability density $p(\bx_0|\bx_t)$ as the main generation process. Besides, we know that the approximation of final outputs using numerical solvers is related to discretization step size $\calO(h_\sigma)$. In our paper, we usually discretize the time range $[t, 0]$ using a uniform step size. Thus, the approximation of final outputs is related to the number of discretization steps $n$. When we use larger steps to solve~\eqref{eqn:appode}, the final solution is closer to the true one, which will make $m$ smaller. Thus, we show the Lemma~\ref{lem:est}.

\subsection{Higher-order symplectic method}
\label{subapp:highorder}

We introduce the first-order (i.e., Euler) symplectic method in Sec.~\ref{sec:meth}. In this part, we introduce the higher-order symplectic method. We present the Symplectic Runge-Kutta method~\cite{matsubara2021symplectic} as an example of a higher-order method. Let $\tau=[n,\ldots, 1]$ denote the discrete steps corresponding to the times from $t$ to $0$ and $\sigma_\tau = \sqrt{1-\alpha_{\tau}}/\sqrt{\alpha_\tau}$. In the Symplectic Runge-Kutta method, we solve the forward ODE~\eqref{eqn:ode1} using the Runge-Kutta solver:
\begin{align}
    \bar{\bx}'_{\sigma_{\tau-1}} &= \bar{\bx}'_{\sigma_{\tau}} + h_{\sigma_{\tau}} \sum_{i=1}^s b_i k_{\sigma_{\tau},i},\notag\\
    k_{\sigma_{\tau},i}:&= \bar{\epsilon}(\bar{X}_{\sigma_{\tau}, i }, \sigma_{\tau} + c_i h_{\sigma_{\tau}}),\notag\\
    \bar{X}_{\sigma_{\tau}, i }:&= \bar{\bx}'_{\sigma_{\tau}} + h_{\sigma_{\tau}} \sum_{j=1}^s a_{i,j} k_{\sigma_{\tau},j}\label{eqn:thm2},
\end{align}
where $a_{i,j}=0$ when $j\geq i$ and the coefficients $a_{i,j}, b_i, c_i$ are summarized as the Butcher tableau~\cite{hairer2006structure}. Then when we solve in the backward direction to obtain the gradients using the Symplectic Runge-Kutta method, we solve the ODE function related to the adjoint state by another Runge–Kutta method with the same step size. It is expressed as
\begin{align}
    \frac{\partial L}{\partial \bar{\bx}'_{\sigma_{\tau}}} &= \frac{\partial L}{\partial \bar{\bx}'_{\sigma_{\tau-1}}} + h_{\sigma_{\tau-1}}\sum_{i=1}^s B_i l_{{\sigma_{\tau-1}},i}, \notag\\
    l_{{\sigma_{\tau-1}},i}:&= -\frac{\partial{\bar{\epsilon}}}{\partial\bar{\bx}'} (\bar{X}_{\sigma_{\tau-1}, i }, \sigma_{\tau-1} + C_i h_{\sigma_{\tau-1}})^T \Lambda_{{\sigma_{\tau-1}},i},\notag\\
    \Lambda_{{\sigma_{\tau-1}},i}:& = \frac{\partial L}{\partial \bar{\bx}'_{\sigma_{\tau-1}}} + h_{\sigma_{\tau-1}}\sum_{j=1}^s A_{i,j} l_{{\sigma_{\tau-1}},j}.\label{eqn:thm3}
\end{align}
The conditions on the parameters are $b_iA_{i,j} + B_j a_{j,i} -b_iB_j=0$ for $i, j = 1,\ldots , s$ and $B_i=b_i\neq 0$ and $C_i=c_i$ for $i=1,\ldots,s$. Besides, the forward solutions $\{\bar{\bx}'_{\sigma_{\tau}}\}_{\tau = 0}^{n}$ needs to save as checkpoints for the backward process.

\subsection{Proof of Theorem~\ref{thm:symp}}
\label{subapp:them}
The Symplectic Euler method we show in Sec.~\ref{sec:meth} is a special case of the higher-order symplectic method when we set $s=1, b_1=1, c_i=0$ in the forward process and set $s=1$ and $b_1=1, B_1=1, a_{1,1}=1, A_{1,1}=0, c_i=C_i=1$ in the backward process. 

To show the formal expression of Theorem~\ref{thm:symp}, we first introduce a variational variable $\delta(\sigma_{\tau})= \frac{\partial \bar{\bx}'_{\sigma_{\tau}}}{\partial \bar{\bx}'_{\sigma_t}}$, which represent the Jacobian of the state $\bar{\bx}'_{\sigma_{\tau}}$ with respect to $\bar{\bx}'_{\sigma_t}$. Denote $\lambda(\sigma_{\tau})=\frac{\partial L}{\partial \bar{\bx}'_{\sigma_\tau}}$ and denote $S(\delta, \lambda)=\lambda^T\delta$.
\begin{theorem} 
\label{thm:symp}
    Let the gradient $\frac{\partial L}{\partial \bar{\bx}'_{\sigma_t}}$  be the analytical solution to the continuous ODE in~\eqref{eqn:symp} and let $\frac{\partial L}{\partial \bar{\bx}'_{\sigma_{n}}}$ be the gradient obtained by the symplectic Euler solver in~\eqref{eqn:sympup2} throughout the discrete sampling process. Then, when $S(\delta,\lambda)$ is conserved (i.e., time-invariant) for the continuous-time system, we have $\frac{\partial L}{\partial \bar{\bx}'_{\sigma_t}} = \frac{\partial L}{\partial \bar{\bx}'_{\sigma_{n}}}$.
\end{theorem}
\begin{proof}
    As we assume $S(\delta, \lambda)$ is conserved for the continuous-time system, we have
    \begin{align*}
        \frac{\rmd }{\rmd \sigma} S(\delta,\lambda) =0. 
    \end{align*}
    Thus we have
    \begin{align*}
        \lambda^T \frac{\rmd \delta}{\rmd \sigma} +  \left(\frac{\rmd \lambda}{\rmd \sigma}\right)^T\delta = 0.
    \end{align*}
    This means that~\cite{matsubara2021symplectic}
    \begin{align*}
       S\left(\frac{\partial k_{\sigma_\tau, i}}{\partial \bar{\bx}_{\sigma_t}},\Lambda_{{\sigma_{\tau}},i} \right) + S\left(\frac{\partial \bar{X}_{\sigma_\tau, i}}{\partial \bar{\bx}_{\sigma_t}}, l_{{\sigma_{\tau}},i}\right)=0
    \end{align*}
    % Then we also have 
    % \begin{align*}
    %     S(\delta(\sigma_{\tau+1}),\lambda(\sigma_{\tau+1})) - S(\delta(\sigma_{\tau}),\lambda(\sigma_{\tau})) 
    % \end{align*}
    Based on~\eqref{eqn:sympup1} and ~\eqref{eqn:sympup2}, we have
    \begin{align*}
        \delta(\sigma_{\tau+1}) &= \delta(\sigma_{\tau}) +h_{\sigma_{\tau}}\frac{\partial k_{\sigma_\tau, 1}}{\partial \bar{\bx}_{\sigma_t}},\\
        % \left(\frac{\partial \bar{\epsilon} (\bar{\bx}'_{\sigma_{\tau+1}}, \sigma_{\tau+1})}{\partial \bar{\bx}'_{\sigma_{\tau+1}}} \right)^T\delta(\sigma_{\tau+1}),\\
        \lambda(\sigma_{\tau+1}) &= \lambda(\sigma_{\tau}) +h_{\sigma_{\tau}}l_{{\sigma_{\tau}},1},
        % \left(\frac{\partial \bar{\epsilon} (\bar{\bx}'_{\sigma_{\tau+1}}, \sigma_{\tau+1})}{\partial \bar{\bx}'_{\sigma_{\tau+1}}} \right)^T\lambda(\sigma_{\tau}),
    \end{align*}
    which means
    \begin{align}
        &S(\lambda(\sigma_{\tau+1}),\delta(\sigma_{\tau+1})) - S(\lambda(\sigma_{\tau}),\delta(\sigma_{\tau}))=\notag\\
        &=S\left(\lambda(\sigma_{\tau}) + h_{\sigma_{\tau}} l_{{\sigma_{\tau}},1}, \delta(\sigma_{\tau})+h_{\sigma_{\tau}} \frac{\partial k_{\sigma_\tau, 1}}{\partial \bar{\bx}_{\sigma_t}}\right)\notag\\
        &\qquad- S(\lambda(\sigma_{\tau}),\delta(\sigma_{\tau}))\notag\\
        &=h_{\sigma_{\tau}} S\left(\lambda(\sigma_{\tau}), \frac{\partial k_{\sigma_\tau, 1}}{\partial \bar{\bx}_{\sigma_t}}\right)+h_{\sigma_{\tau}}S(\delta(\sigma_{\tau}), l_{{\sigma_{\tau}},1})\notag\\
        &\qquad + h_{\sigma_{\tau}}^2 S\left(\frac{\partial k_{\sigma_\tau, 1}}{\partial \bar{\bx}_{\sigma_t}},l_{{\sigma_{\tau}},1} \right)\notag\\
        &\overset{(a)}{=} h_{\sigma_{\tau}} S\left(\Lambda_{{\sigma_{\tau}},i} , \frac{\partial k_{\sigma_\tau, 1}}{\partial \bar{\bx}_{\sigma_t}}\right) \notag\\
        &\qquad+ h_{\sigma_{\tau}}S\left(\frac{\partial \bar{X}_{\sigma_\tau, 1}}{\partial \bar{\bx}_{\sigma_t}} -h_{\sigma_{\tau}} \frac{\partial k_{\sigma_\tau, 1}}{\partial \bar{\bx}_{\sigma_t}}, l_{{\sigma_{\tau}},1}\right) \notag\\
        &\qquad + h_{\sigma_{\tau}}^2 S\left(\frac{\partial k_{\sigma_\tau, 1}}{\partial \bar{\bx}_{\sigma_t}},l_{{\sigma_{\tau}},1} \right)\notag\\
        &=0,\label{eqn:thm1}
    \end{align}
    where the first term of (a) is based on~\eqref{eqn:thm3} and the second term of (a) is based on~\eqref{eqn:thm2}. Thus, we have
    \begin{align}
    \label{eqn:thm2}
        \lambda(\sigma_{n})^T\delta(\sigma_{n}) \overset{(a)}{=}\lambda(\sigma_{0})^T\delta(\sigma_{0})\overset{(b)}{=}\lambda(\sigma_{t})^T\delta(\sigma_{t}),
    \end{align}
    where $(a)$ is based on~\eqref{eqn:thm1} and (b) is based on our assumption that $S(\delta,\lambda)$ is conserved for the continuous-time system. Then based on~\eqref{eqn:thm2} and $\bar{\bx}'_{\sigma_n}=\bar{\bx}'_{\sigma_t}$, we have 
    \begin{align*}
        \frac{\partial L(\bar{\bx}'_{\sigma_{0}})}{\partial \bar{\bx}'_{\sigma_{n}}} &=\frac{\partial L(\bar{\bx}'_{\sigma_{0}})}{\partial \bar{\bx}'_{\sigma_{0}}} \frac{\partial \bar{\bx}'_{\sigma_{0}}}{\partial \bar{\bx}'_{\sigma_{n}}} \\
        &= \lambda(\sigma_{0})^T\delta(\sigma_{0})\\
        &=\lambda(\sigma_{t})^T\delta(\sigma_{t})\\
        &=\frac{\partial L(\bar{\bx}'_{\sigma_{0}})}{\partial \bar{\bx}'_{\sigma_{t}}},
    \end{align*}
    which proves our theorem.
\end{proof}

\section{Experimental Details on Guided Sampling in Image Generation}
\label{app:exp}

\subsection{Style-Guided Generation}
\label{subapp:style}

In this section, we introduce the experimental details of style-guided sampling. Let sampling times be $T=100$. We set to do SAG from sampling steps $t=70$ to $t=31$ and the repeats time from $70$ and $61$ as 1 and from $60$ to $31$ as 2. For quantitative results, we select five style images and choose four prompts: \emph{["A cat wearing glasses.",  "A fantasy photo of volcanoes.", "A photo of an Eiffel Tower.",  "butterfly"]} to generate five images per prompt per style. 
For the implementation of Universal Guidance~\cite{Bansal_2023_CVPR} and FreeDOM~\cite{Yu2023freedom}, we use the officially released codes and generate the results for quantitative comparison under sample style and prompts. Besides, the hyperparameter choice for these two models also follows the official implementations. More qualitative results are shown in Fig.~\ref{figapp:style}.

\subsection{Aesthetic Improvement}
\label{subapp:aes}

When we improve the aesthetics of generated images, we use the weighted losses for LAION aesthetic predictor,\footnote{\url{https://github.com/LAION-AI/aesthetic-predictor.git}} PickScore~\cite{kirstain2023pick} and HPSv2~\cite{wu2023human}. We set the weights for each aesthetic evaluation model as \emph{PickScore = 10, HPSv2 = 2, Aesthetic = 0.5}. Let sampling times be $T=100$. We set to do SAG from sampling steps $t=70$ to $t=31$ and the repeats time from $70$ and $41$ as 2 and from $40$ to $31$ as 1. More qualitative results are shown in Fig.~\ref{figapp:aes}.

\subsection{Personalization}
\label{subapp:person}

For personalization in the object-guided generation, we do training-free guidance from steps $t=100$ to $t=31$ and we set the repeat times as 2. We randomly select four reference dog images and select four prompts: \emph{A dog (at the Acropolis/swimming/in a
bucket/wearing sunglasses)}. We generate four images per prompt per image to measure the quantitative results. For the results of DOODL, we directly use the results in the paper~\cite{wallace2023endtoend}. For the results of FreeDOM, we use the special case of our model when we set $n=1$. Let sampling times be $T=100$. We set to do SAG from sampling steps $t=100$ to $t=31$ and the repeats time from $100$ and $31$ as 2. More qualitative results are shown in Fig.~\ref{figapp:object}.

\subsection{Ablation Study}
\label{subapp:abla}

\paragraph{Analyses on Memory and Time Consumption}

We conducted our experiments on a V100 GPU. Memory consumption using SAG was observed to be 15.66GB, compared to 15.64GB when employing ordinary adjoint guidance. Notably, direct gradient backpropagation at $n=2$ resulted in a significantly higher memory usage of 28.63GB. Furthermore, as $n$ increases, the memory requirement for direct backpropagation shows a corresponding increase. In contrast, when using SAG, the memory consumption remains nearly constant regardless of the value of $n$.

We also present the time consumption associated with a single step of SAG for varying values of $n$ in Fig.~\ref{figapp:time}. As $n$ increases, we observe a corresponding rise in time consumption. However, this increment in $n$ also results in a substantial reduction in loss as shown in Fig.~\ref{fig:n_loss}, indicating a trade-off between computational time and the quality of results. 
\begin{figure}
    \centering
    \includegraphics[width=0.8\linewidth]{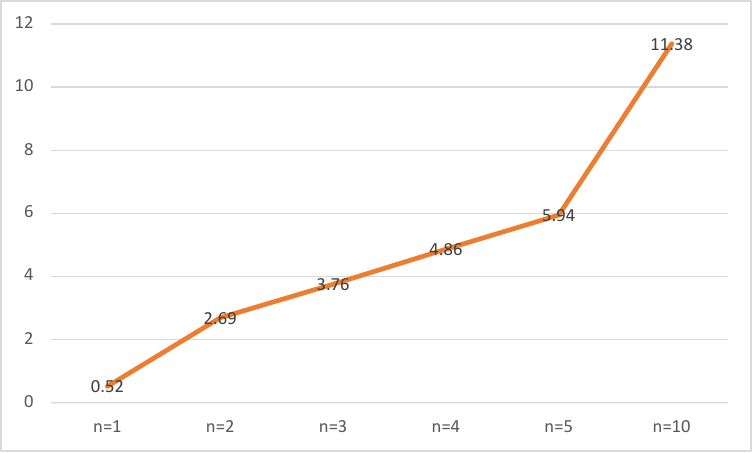}
    \caption{Time consumption of single step SAG (seconds)}
    \label{figapp:time}
\end{figure}

\paragraph{Choice of repeat times of time travel}

We show some results about the choice of repeat times in Fig.~\ref{figapp:repeat}. We find that increasing the repeat times helps the stylization. Besides, there still exists the distortion of images when $n=1$ even when we increase the repeat times. 
\begin{figure}
    \centering
    \includegraphics[width=\linewidth]{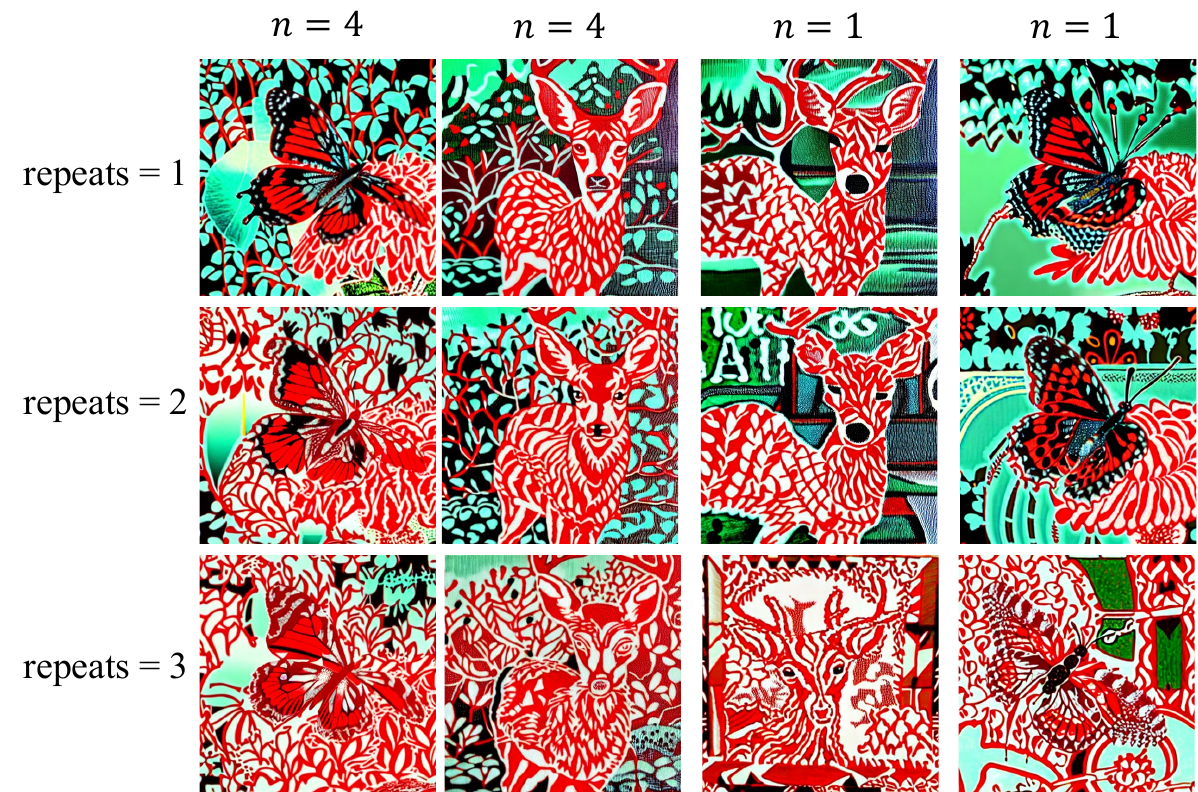}
    \caption{Stylization results when we use different repeat times of time travel.}
    \label{figapp:repeat}
\end{figure}

\paragraph{Choice of guidance steps}
We present the qualitative results regarding the selection of guidance steps in Fig.~\ref{figapp:start}. We can observe that initiating guidance in the early stages (i.e., the chaotic stage) results in final outputs that differ from those generated without guidance. Besides, starting guidance in the semantic stage allows us to maintain the integrity of the original images while effectively achieving style transfer. 
\begin{figure}
    \centering
    \includegraphics[width=0.8\linewidth]{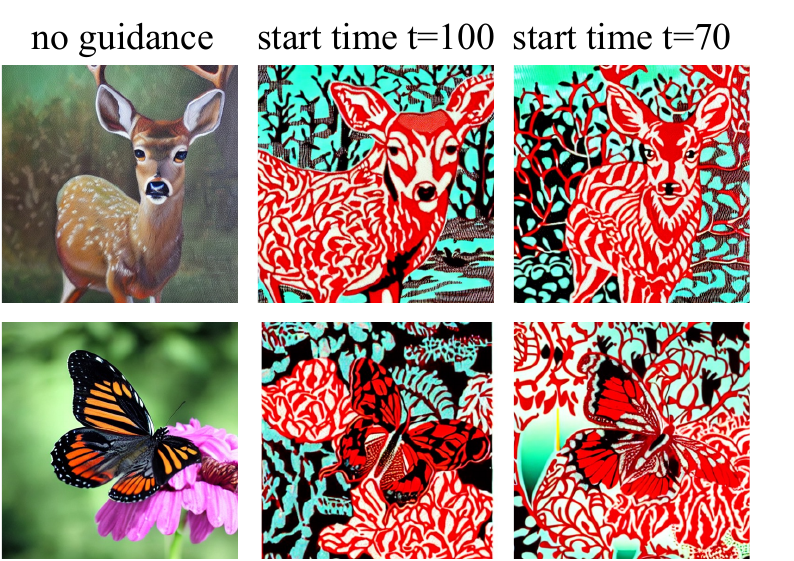}
    \caption{Stylization results when we start to do guidance at different time steps.}
    \label{figapp:start}
\end{figure}

\section{More examples on Video Stylization}
\label{app:video}
Two more groups of results of style-guided video stylization are shown in Figure~\ref{fig:apdx-video}.
\begin{figure}[h!]
    \centering
    \includegraphics[width=\linewidth]{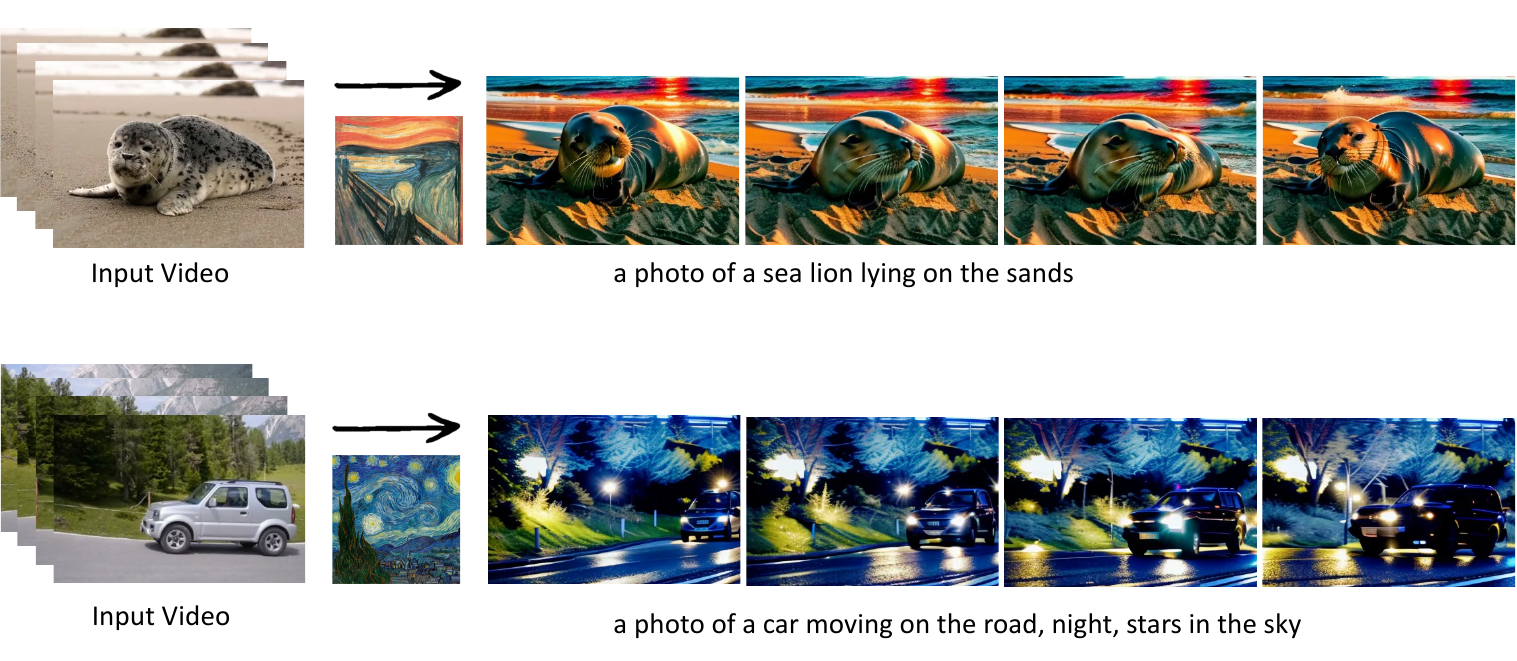}
    \caption{Examples on Video Stylization.}
    \label{fig:apdx-video}
\end{figure}

\section{Additional Qualitative Results on Image Generation}
\label{app:qual}
\begin{figure}
    \centering
    \includegraphics[width=\linewidth]{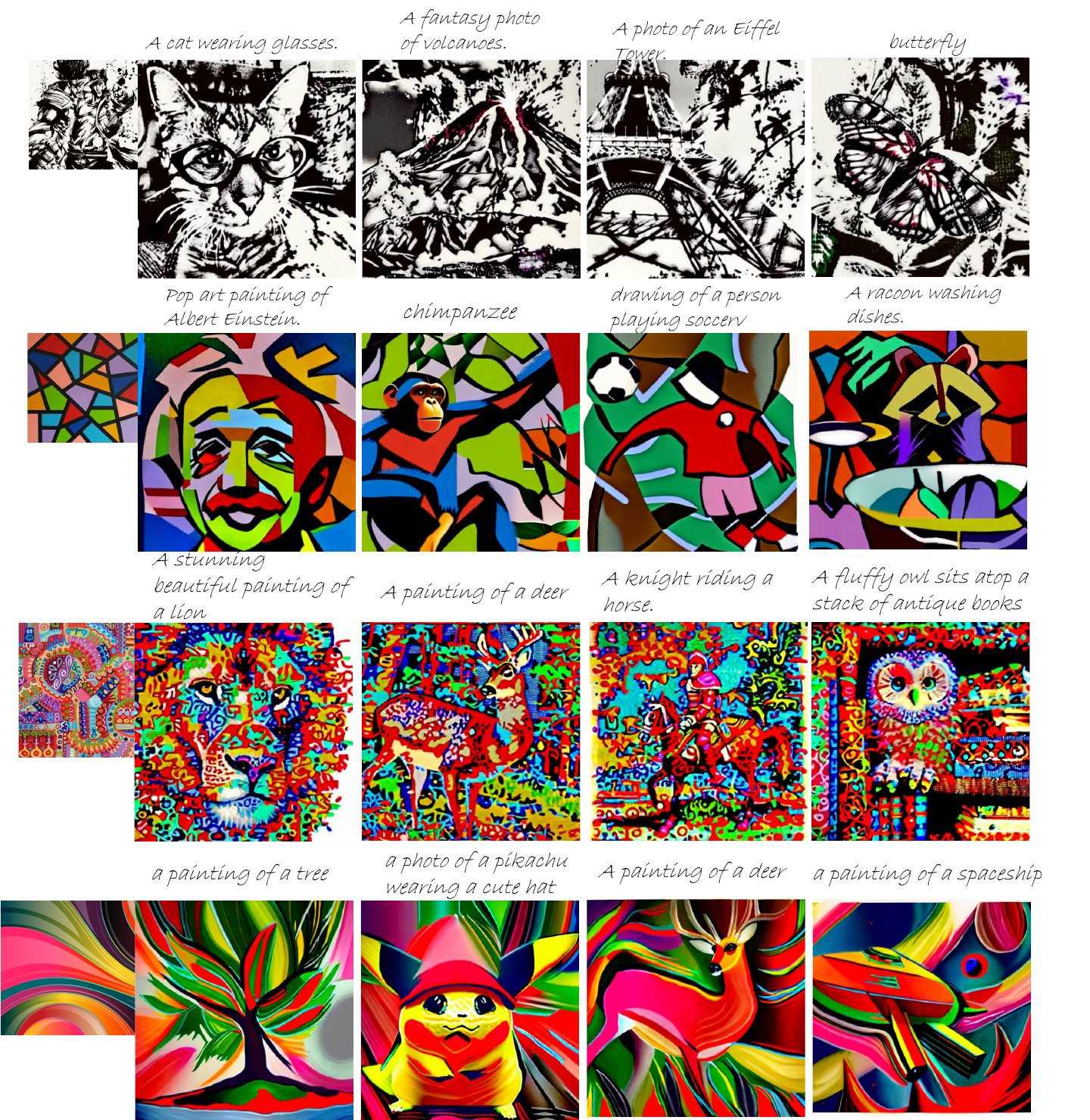}
    \caption{More examples of style-guided generation.}
    \label{figapp:style}
\end{figure}

\begin{figure*}
    \centering
    \includegraphics[width=\linewidth]{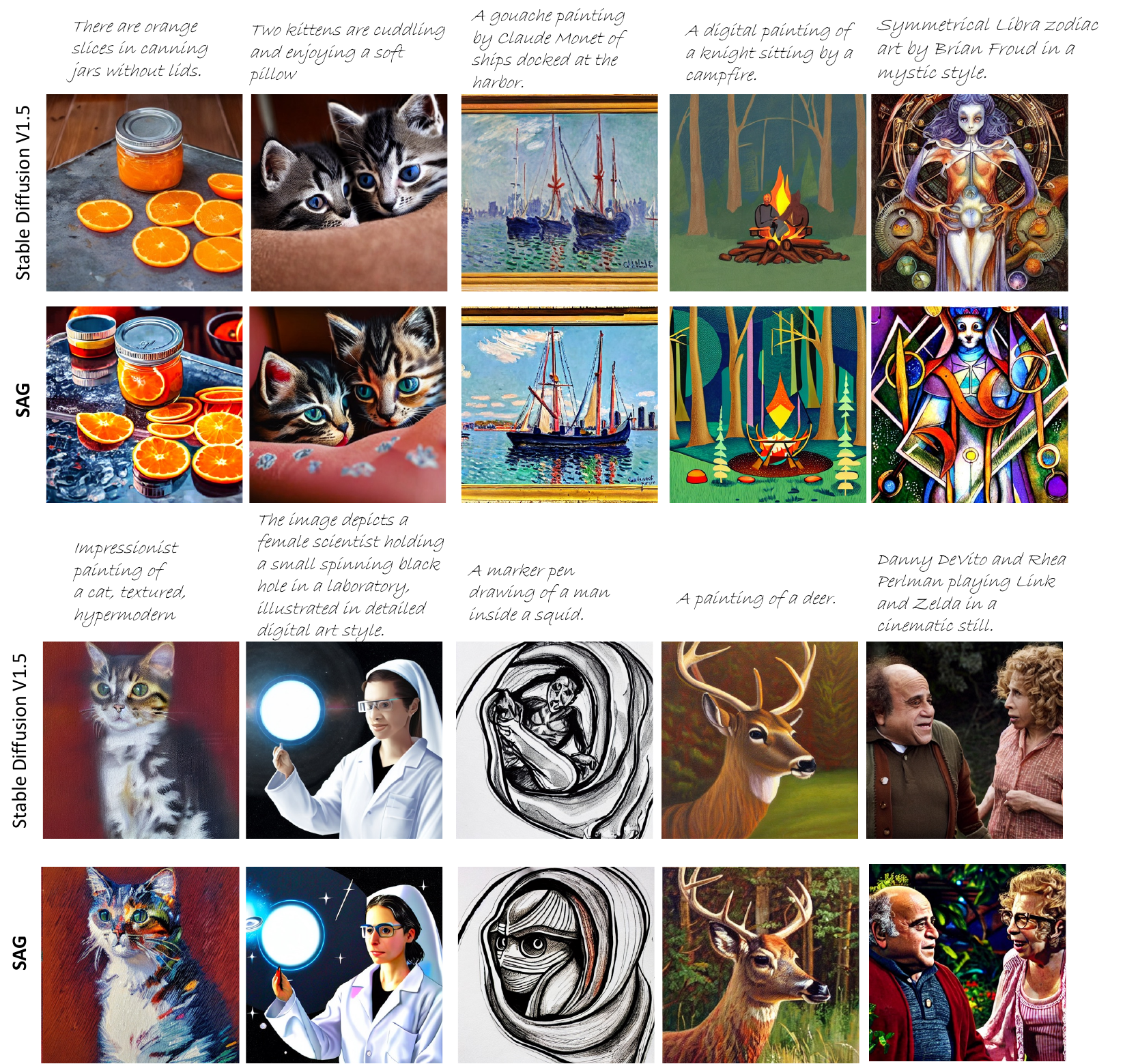}
    \caption{More examples of aesthetic improvements.}
    \label{figapp:aes}
\end{figure*}

\begin{figure*}
\centering
    \begin{subfigure}{0.49\textwidth}
        \includegraphics[width=\linewidth]{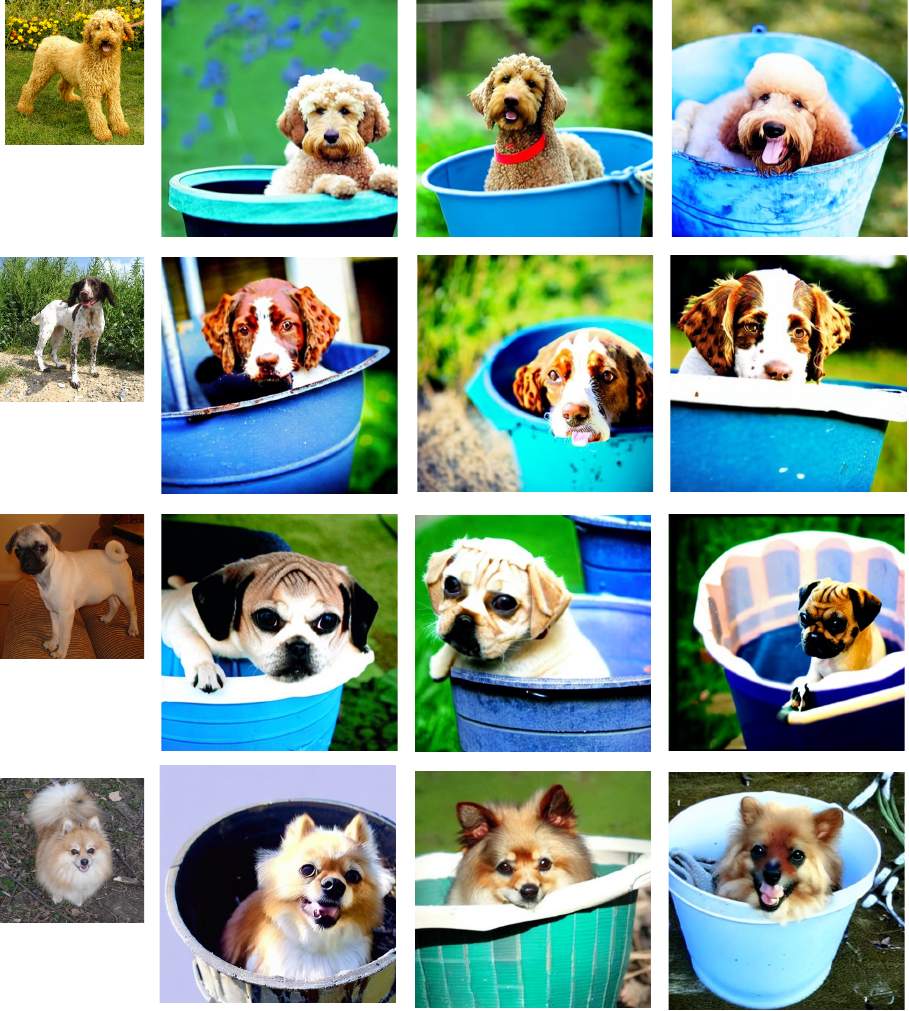}
        \caption{\emph{A dog in the bucket}.}
    \end{subfigure} 
    \begin{subfigure}{0.40\textwidth}
        \includegraphics[width=\linewidth]{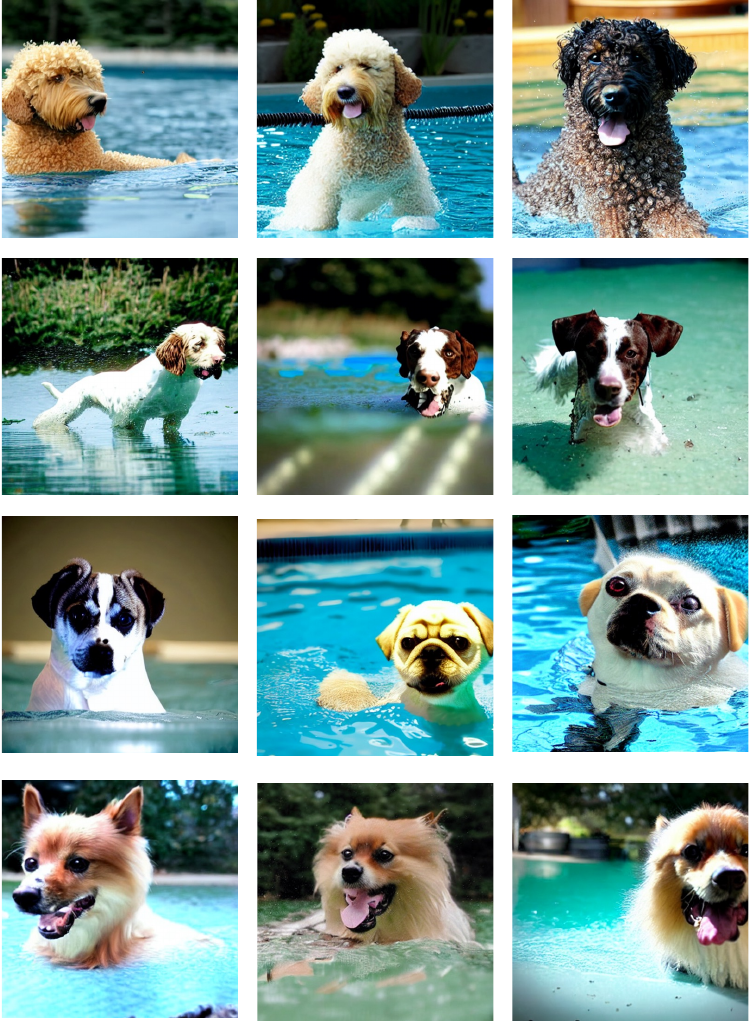}
        \caption{\emph{A dog swimming}.}
    \end{subfigure}

    \begin{subfigure}{0.49\textwidth}
        \includegraphics[width=\linewidth]{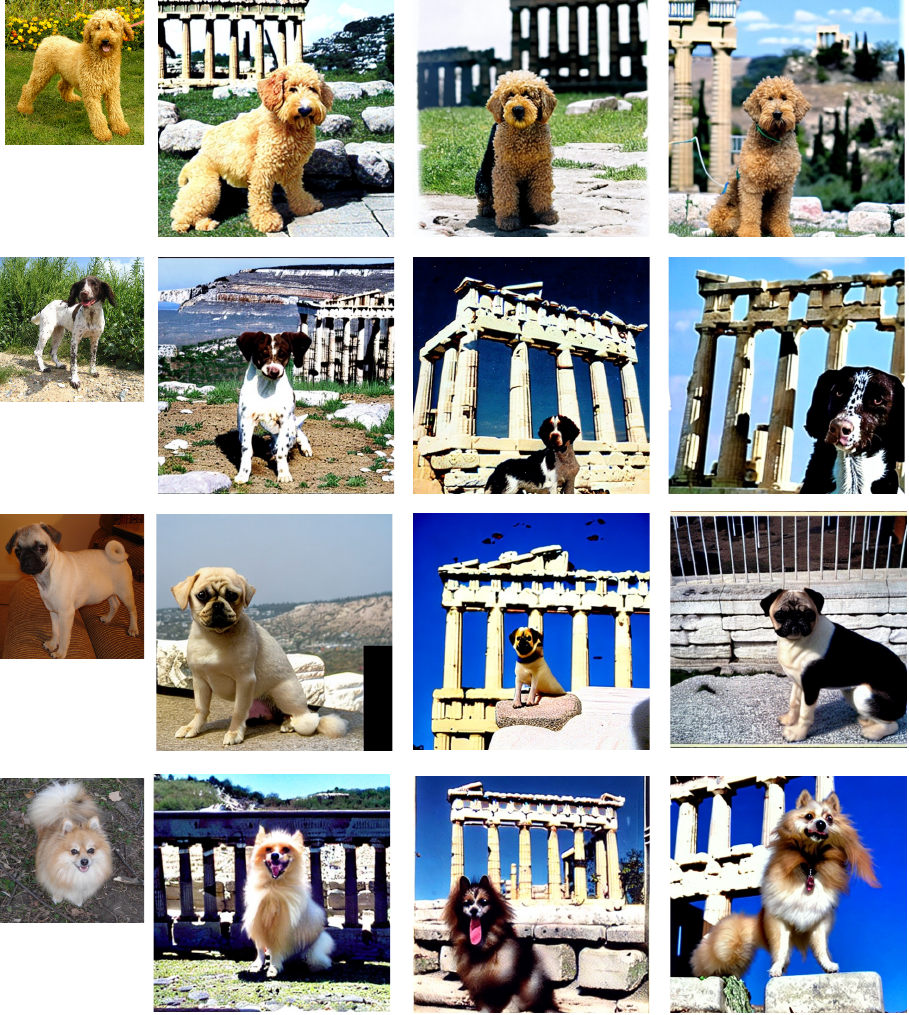}
        \caption{\emph{A dog at Acropolis}.}
    \end{subfigure} 
    \begin{subfigure}{0.40\textwidth}
        \includegraphics[width=\linewidth]{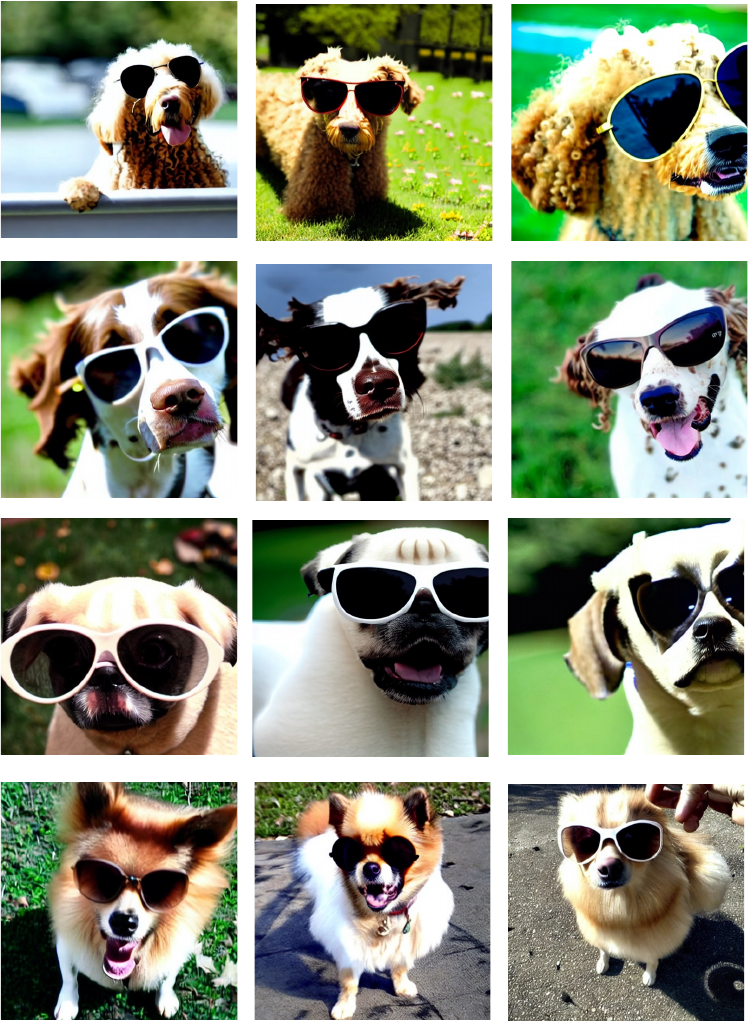}
        \caption{\emph{A dog wearing sunglasses}.}
    \end{subfigure}
    \caption{More examples on object-guided personalization.}
    \label{figapp:object}
\end{figure*}


\end{document}